%% file: main.tex
\documentclass{article}

\usepackage[final]{neurips_2019}

\usepackage[utf8]{inputenc} %
\usepackage[T1]{fontenc}    %
\usepackage{hyperref}       %
\usepackage{url}            %
\usepackage{booktabs}       %
\usepackage{amsfonts}       %
\usepackage{nicefrac}       %
\usepackage{microtype}      %
\usepackage{amsmath}
\usepackage{amssymb}
\usepackage{enumerate}
\usepackage{graphicx}
\usepackage{subcaption}
\usepackage{algorithm}
\usepackage{algorithmic}
\usepackage{bm}
\usepackage{xcolor}
\usepackage{color}
\usepackage{amsthm}
\usepackage{tikz}
\usepackage[font=small,labelfont=bf]{caption}
\newcommand*\circled[1]{\tikz[baseline=(char.base)]{
   \node[shape=circle,draw,inner sep=1pt] (char) {#1};}}

\definecolor{mydarkblue}{rgb}{0,0.1,0.45}
\hypersetup{
    colorlinks=true,
    linkcolor=mydarkblue,
    citecolor=mydarkblue,
    filecolor=mydarkblue,
    urlcolor=mydarkblue}

\newtheorem{thm}{Theorem}
\newtheorem{lem}{Lemma}

\newtheorem{deff}{Definition}
\newtheorem{rem}{Remark}
\newtheorem{con}{Condition}
\newtheorem{ass}{Assumption}

\newcommand{\real}{\mathbb{R}}
\newcommand{\sphere}{\mathbb{S}}
\newcommand{\prob}{\mathbb{P}}
\newcommand{\expect}{\mathbb{E}}
\newcommand{\indicator}{\mathbb{I}}
\newcommand{\rade}{\mathcal{R}}
\newcommand{\loss}{\mathcal{L}}
\newcommand{\bigo}{\mathcal{O}}
\newcommand{\normal}{\mathcal{N}}
\newcommand{\hilbert}{\mathcal{H}}
\newcommand{\data}{\mathcal{D}}
\newcommand{\tdata}{\mathcal{S}}
\newcommand{\fisher}{\mathbf{F}}

\newcommand{\hessian}{\mathbf{H}}
\newcommand{\jacobian}{\mathbf{J}}
\newcommand{\gram}{\mathbf{G}}
\newcommand{\iden}{\mathbf{I}}
\newcommand{\weight}{\mathbf{w}}
\newcommand{\params}{\bm{\theta}}
\newcommand{\inputs}{\mathbf{x}}

\newcommand{\pred}{\mathbf{u}}
\newcommand{\target}{\mathbf{y}}
\newcommand{\error}{\mathbf{e}}
\newcommand{\lr}{\eta}
\newcommand{\vars}{\nu}
\newcommand{\relu}{\phi}
\newcommand{\trace}{\mathrm{tr}}
\newcommand{\bz}{\mathbf{z}}
\newcommand{\bX}{\mathbf{X}}
\newcommand{\bS}{\mathbf{S}}
\newcommand{\bA}{\mathbf{A}}
\newcommand{\bB}{\mathbf{B}}
\newcommand{\bZ}{\mathbf{Z}}
\newcommand{\bT}{\mathbf{T}}
\newcommand{\ba}{\mathbf{a}}

\title{Fast Convergence of Natural Gradient Descent \\ for Overparameterized Neural Networks}

\author{
Guodong Zhang${}^{1, 2}$,  James Martens${}^{3}$,  Roger Grosse${}^{1, 2}$\\
University of Toronto${}^{1}$,  Vector Institute${}^{2}$,  DeepMind${}^{3}$\\
\texttt{\{gdzhang, rgrosse\}@cs.toronto.edu}, \texttt{jamesmartens@google.com}
}

\begin{document}

\maketitle

\begin{abstract}
Natural gradient descent has proven effective at mitigating the effects of pathological curvature in neural network optimization, but little is known theoretically about its convergence properties, especially for \emph{nonlinear} networks. In this work, we analyze for the first time the speed of convergence of natural gradient descent on nonlinear neural networks with squared-error loss. We identify two conditions which guarantee efficient convergence from random initializations: (1) the Jacobian matrix (of network's output for all training cases with respect to the parameters) has full row rank, and (2) the Jacobian matrix is stable for small perturbations around the initialization. For two-layer ReLU neural networks, we prove that these two conditions do in fact hold throughout the training, under the assumptions of nondegenerate inputs and overparameterization. We further extend our analysis to more general loss functions. Lastly, we show that K-FAC, an approximate natural gradient descent method, also converges to global minima under the same assumptions, and we give a bound on the rate of this convergence.
\end{abstract}

\section{Introduction}

Because training large neural networks is costly, there has been much interest in using second-order optimization to speed up training \citep{becker1989improving,martens2010deep,martens2015optimizing}, and in particlar natural gradient descent~\citep{amari1998natural, amari1997neural}. Recently, scalable approximations to natural gradient descent have shown practical success in a variety of tasks and architectures~\citep{martens2015optimizing,grosse2016kronecker,wu2017scalable,zhang2018noisy,martens2018kronecker}. Natural gradient descent has an appealing interpretation as optimizing over a Riemannian manifold using an intrinsic distance metric; this implies the updates are invariant to transformations such as whitening \citep{ollivier2015riemannian,luk2018coordinate}. It is also closely connected to Gauss-Newton optimization, suggesting it should achieve fast convergence in certain settings~\citep{pascanu2013revisiting,martens2014new,botev2017practical}.

Does this intuition translate into faster convergence? \citet{amari1998natural} provided arguments in the affirmative, as long as the cost function is well approximated by a convex quadratic.
However, it remains unknown whether natural gradient descent can optimize neural networks faster than gradient descent --- a major gap in our understanding. The problem is that the optimization of neural networks is both nonconvex and non-smooth, making it difficult to prove nontrivial convergence bounds. In general, finding a global minimum of a general non-convex function is an NP-complete problem, and neural network training in particular is NP-complete~\citep{Blum1992TrainingA3}.

However, in the past two years, researchers have finally gained substantial traction in understanding the dynamics of gradient-based optimization of neural networks. Theoretically, it has been shown that gradient descent starting from a random initialization is able to find a global minimum if the network is wide enough~\citep{li2018learning, du2018gradient, du2018gradient-2, zou2018stochastic, allen2018convergence, oymak2019towards}. The key technique of those works is to show that neural networks become well-behaved if they are largely overparameterized in the sense that the number of hidden units is polynomially large in the size of the training data. However, most of these works have focused on standard gradient descent, leaving open the question of whether similar statements can be made about other optimizers.

Most convergence analysis of natural gradient descent has focused on simple convex quadratic objectives (e.g.~\citep{martens2014new}).
Very recently, the convergence properties of NGD were studied in the context of linear networks~\citep{bernacchia2018exact}. While the linearity assumption simplifies the analysis of training dynamics~\citep{saxe2013exact}, linear networks are severely limited in terms of their expressivity, and it’s not clear which conclusions will generalize from linear to nonlinear networks.

In this work, we analyze natural gradient descent for \emph{nonlinear networks}. We give two simple and generic conditions on the Jacobian matrix which guarantee efficient convergence to a global minimum. We then apply this analysis to a particular distribution over two-layer ReLU networks which has recently been used to analyze the convergence of gradient descent~\citep{li2018learning, du2018gradient-2, oymak2019towards}. We show that for sufficiently high network width, NGD will converge to the global minimum. We give bounds on the convergence rate of two-layer ReLU networks that are much better than the analogous bounds that have been proven for gradient descent~\citep{du2018gradient, wu2019global, oymak2019towards}, while allowing for much higher learning rates. Moreover, in the limit of infinite width, and assuming a squared error loss, we show that NGD converges in just \emph{one iteration}.
The main contributions of our work are summarized as follows:
\begin{itemize}
    \setlength\itemsep{0.05em}
    \item We provide the first convergence result for natural gradient descent in training randomly-initialized overparameterized neural networks where the number of hidden units is polynomially larger than the number of training samples. We show that natural gradient descent gives an $\bigo(\lambda_\mathrm{min}(\gram^\infty))$ improvement in convergence rate given the same learning rate as gradient descent, where $\gram^\infty$ is a Gram matrix that depends on the data.
    \item We show that natural gradient enables us to use a much larger step size, resulting in an even faster convergence rate. Specifically, the maximal step size of natural gradient descent is $\bigo\left(1 \right)$ for (polynomially) wide networks.
    \item We show that K-FAC \citep{martens2015optimizing}, an approximate natural gradient descent method, also converges to global minima with linear rate, although this result requires a higher level of overparameterization compared to GD and exact NGD.
    \item We analyze the generalization properties of NGD, showing that the improved convergence rates \emph{provably} don't come at the expense of worse generalization.
\end{itemize}

\section{Related Works}

Recently, there have been many works studying the optimization problem in deep learning, i.e., why in practice many neural network architectures reliably converge to global minima (zero training error). One popular way to attack this problem is to analyze the underlying loss surface~\citep{hardt2016identity, kawaguchi2016deep, kawaguchi2018depth, nguyen2017loss, soudry2016no}. The main argument of those works is that there are no bad local minima. It has been proven that gradient descent can find global minima~\citep{ge2015escaping, lee2016gradient} if the loss surface satisfies: (1) all local minima are global and (2) all saddle points are strict in the sense that there exists at least one negative curvature direction. Unfortunately, most of those works rely on unrealistic assumptions (e.g., linear activations~\citep{hardt2016identity, kawaguchi2016deep}) and cannot generalize to practical neural networks. Moreover, \citet{yun2018small} shows that small nonlinearity in shallow networks can create bad local minima.

Another way to understand the optimization of neural networks is to directly analyze the optimization dynamics. Our work also falls within this category. However, most work in this direction focuses on gradient descent. \citet{bartlett2019gradient, arora2018a} studied the optimization trajectory of deep linear networks and showed that gradient descent can find global minima under some assumptions. Previously, the dynamics of linear networks have also been studied by \citet{saxe2013exact, advani2017high}. For nonlinear neural networks, a series of papers~\citep{tian2017analytical, brutzkus2017globally, du2017convolutional, li2017convergence, zhang2018learning} studied a specific class of shallow two-layer neural networks together with strong assumptions on input distribution as well as realizability of labels, proving global convergence of gradient descent. Very recently, there are some works proving global convergence of gradient descent~\citep{li2018learning, du2018gradient, du2018gradient-2, allen2018convergence, zou2018stochastic, gao2019learning} or adaptive gradient methods~\citep{wu2019global} on overparameterized neural networks. More specifically, \citet{li2018learning, allen2018convergence, zou2018stochastic} analyzed the dynamics of weights and showed that the gradient cannot be small if the objective value is large. On the other hand, \citet{du2018gradient, du2018gradient-2, wu2019global} studied the dynamics of the outputs of neural networks, where the convergence properties are captured by a Gram matrix. Our work is very similar to~\citet{du2018gradient, wu2019global}. We note that these papers all require the step size to be sufficiently small to guarantee the global convergence, leading to slow convergence.

To our knowledge, there is only one paper~\citep{bernacchia2018exact} studying the global convergence of natural gradient for neural networks. However, \citet{bernacchia2018exact} only studied deep linear networks with infinitesimal step size and squared error loss functions. In this sense, our work is the first one proving global convergence of natural gradient descent on nonlinear networks. 

There have been many attempts to understand the generalization properties of neural networks since \citet{zhang2016understanding}'s seminal paper. Researchers have proposed norm-based generalization bounds~\citep{neyshabur2015norm, neyshabur2017pac, bartlett2002rademacher, bartlett2017spectrally, golowich2017size}, compression bounds~\citep{arora2018stronger} and PAC-Bayes bounds~\citep{dziugaite2017computing, dziugaite2018data, zou2018stochastic}. Recently, overparameterization of neural networks together with good initialization has been believed to be one key factor of good generalization. \citet{neyshabur2018the} empirically showed that wide neural networks stay close to the initialization, thus leading to good generalization. Theoretically, researchers did prove that overparameterization as well as linear convergence jointly restrict the weights to be close to the initialization~\citep{du2018gradient, du2018gradient-2, allen2018convergence, zou2018stochastic, arora2019fine}. The most closely related paper is \citet{arora2019fine}, which shows that the optimization and generalization phenomenon can be explained by a Gram matrix. The main difference is that our analysis is based on natural gradient descent, which converges faster and provably generalizes as well as gradient descent. 

Concurrently and independently, \citet{cai2019gram} showed that natural gradient descent (they call it Gram-Gauss-Newton) enjoys quadratic convergence rate guarantee for overparameterized networks on regression problems. Additionally, they showed that it is much cheaper to precondition the gradient in the output space when the number of data points is much smaller than the number of parameters.

\section{Convergence Analysis of Natural Gradient Descent}\label{sec:ngd}
We begin our convergence analysis of natural gradient descent -- under appropriate conditions -- for the neural network optimization problem. Formally, we consider a generic neural network $f(\params, \inputs)$ with a single output and squared error loss $\ell(u, y) = \frac{1}{2}(u - y)^2$ for simplicity\footnote{It is easy to extend to multi-output networks and other loss functions, here we focus on single-output and quadratic just for notational simplicity.}, where $\params \in \real^m$ denots all parameters of the network (i.e.~weights and biases). Given a training dataset $\left\{\left(\inputs_i, y_i \right) \right\}_{i=1}^n$, we want to minimize the following loss function:
\begin{equation}
    \loss(\params) = \frac{1}{n} \sum_{i=1}^n \ell\left(f(\params, \inputs_i), y_i\right) = \frac{1}{2n} \sum_{i=1}^n \left(f(\params, \inputs_i) - y_i \right)^2 .
\end{equation}
One main focus of this paper is to analyze the following procedure:
\begin{equation}\label{eq:ngd-update}
    \params(k+1) = \params(k) - \lr \fisher(\params(k))^{-1} \frac{\partial\loss(\params(k))}{\partial \params(k)} ,
\end{equation}
where $\lr >0$ is the step size, and $\fisher$ is the Fisher information matrix associated with the network's predictive distribution over $y$ (which is implied by its loss function and is $\normal(f(\params, \inputs_i), 1)$ for the squared error loss) and the dataset's distribution over $\inputs$.

As shown by~\citet{martens2014new}, the Fisher $\fisher$ is equivalent to the generalized Gauss-Newton matrix, defined as $\expect_{\inputs_i} \left[\jacobian_i^\top \hessian_\ell \jacobian_i \right]$ if the predictive distribution is in the exponential family, such as categorical distribution (for classification) or Gaussian distribution (for regression). $\jacobian_i$ is the Jacobian matrix of $\pred_i$ with respect to the parameters $\params$ and $\hessian_\ell$ is the Hessian of the loss $\ell(\pred, \target)$ with respect to the network prediction $\pred$ (which is $\iden$ in our setting). Therefore, with the squared error loss, the Fisher matrix can be compactly written as $\fisher = \expect \left[\jacobian_i^\top \jacobian_i \right] = \frac{1}{n} \jacobian^\top \jacobian$ (which coincides with classical Gauss-Newton matrix), where $\jacobian = [\jacobian_1^\top, ..., \jacobian_n^\top]^\top$ is the Jacobian matrix for the whole dataset. 
In practice, when the number of parameters $m$ is larger than number of samples $n$ we have, the Fisher information matrix $\fisher = \frac{1}{n} \jacobian^\top \jacobian$ is surely singular. In that case, we take the generalized inverse~\citep{bernacchia2018exact} $\fisher^{\dagger} = n \jacobian^\top \gram^{-1} \gram^{-1} \jacobian$ with $\gram = \jacobian \jacobian^\top$, which gives the following update rule:
\begin{equation}
    \params(k+1) = \params(k) - \lr \jacobian^\top \left(\jacobian \jacobian^\top \right)^{-1} (\pred - \target) ,
\end{equation}
where $\pred = [\pred_1, ..., \pred_n]^\top = [f(\params, \inputs_1), ..., f(\params, \inputs_n)]^\top$ and $\target = [y_1, ..., y_n]^\top$.

We now introduce two conditions on the network $f_{\params}$ that suffice for proving the global convergence of NGD to a minimizer which achieves zero training loss (and is therefore a global minimizer). To motivate these two conditions we make the following observations. First, the global minimizer is characterized by the condition that the gradient in the output space is zero for each case (i.e. $\nabla_{\pred}\loss(\params) = \mathbf{0}$). Meanwhile, local minima are characterized by the condition that the gradient with respect to the parameters $\nabla_{\params} \loss(\params)$ is zero. Thus, one way to avoid finding local minima that aren't global is to ensure that the parameter gradient is zero if and only if the output space gradient (for each case) is zero. It's not hard to see that this property holds as long as $\gram$ remains non-singular throughout optimization (or equivalently that $\jacobian$ always has full row rank). The following two conditions ensure that this happens, by first requiring that this property hold at initialization time, and second that $\jacobian$ changes slowly enough that it remains true in a big enough neighborhood around $\params(0)$.
\begin{con}[Full row rank of Jacobian matrix]\label{con:full-row-rank}
    The Jacobian matrix $\jacobian(0)$ at the initialization has full row rank, or equivalently, the Gram matrix $\gram(0) = \jacobian(0) \jacobian(0)^\top$ is positive definite.
\end{con}
\begin{rem}
    Condition~\ref{con:full-row-rank} implies that $m \leq n$, which means the Fisher information matrix is singular and we have to use the generalized inverse except in the case where $m = n$.
\end{rem}
\begin{con}[Stable Jacobian]\label{con:stable-jaco}
    There exists $0 \leq C < \frac{1}{2}$ such that for all parameters $\params$ that satisfy $\|\params - \params(0) \|_2 \leq \frac{3\|\target - \pred(0)\|_2}{\sqrt{\lambda_{\mathrm{min}}(\gram(0))}}$, we have
    \begin{equation}
        \|\jacobian(\theta) - \jacobian(0)\|_2 \leq \frac{C}{3} \sqrt{\lambda_{\mathrm{min}}(\gram(0))} .
    \end{equation}
\end{con}
This condition shares the same spirit with the Lipschtiz smoothness assumption in classical optimization theory. 
It implies (with small $C$) that the network is close to a linearized network~\citep{lee2019wide} around the initialization and therefore natural gradient descent update is close to the gradient descent update in the output space.
Along with Condition~\ref{con:full-row-rank}, we have the following theorem.
\begin{thm}[Natural gradient descent]\label{thm:ngd}
Let Condition~\ref{con:full-row-rank} and~\ref{con:stable-jaco} hold.  Suppose we optimize with NGD using a step size $\lr \leq \frac{1-2C}{(1+C)^2}$. Then for $k = 0, 1, 2, ...$ we have
\begin{equation}
    \left\| \pred(k) - \target \right\|_2^2 \leq (1 - \lr)^k \left\| \pred(0) - \target \right\|_2^2 .
\end{equation}
\end{thm}
To be noted, $\left\| \pred(k) - \target \right\|_2^2$ is the squared error loss up to a constant. 
Due to space constraints we only give a short sketch of the proof here. The full proof is given in Appendix \ref{app:main_proof}.

\textbf{Proof Sketch}. \emph{Our proof relies on the following insights. First, if the Jacobian matrix has full row rank, this guarantees linear convergence for infinitesimal step size. The linear convergence property restricts the parameters to be close to the initialization, which implies the Jacobian matrix is always full row rank throughout the training, and therefore natural gradient descent with infinitesimal step size converges to global minima. Furthermore, given the network is close to a linearized network (since the Jacobian matrix is stable with respect to small perturbations around the initialization), we are able to extend the proof to discrete time with a large step size.}

In summary, we prove that NGD exhibits linear convergence to the global minimizer of the neural network training problem, under Conditions~\ref{con:full-row-rank} and~\ref{con:stable-jaco}. We believe our arguments in this section are general (i.e., architecture-agnostic), and can serve as a recipe for proving global convergence of natural gradient descent in other settings.

\subsection{Other Loss Functions}
We note that our analysis can be easily extended to more general loss function class. Here, we take the class of functions that are $\mu$-strongly convex with $L$-Lipschitz gradients as an example. Note that strongly convexity is a very mild assumption since we can always add $L_2$ regularization to make the convex loss strongly convex. Therefore, this function class includes regularized cross-entropy loss (which is typically used in classification) and squared error (for regression). For this type of loss, we need a strong version of Condition~\ref{con:stable-jaco}.
\begin{con}[Stable Jacobian]\label{con:stable-jaco-2}
	There exists $0 \leq C < \frac{1}{1+\kappa}$ such that for all parameters $\params$ that satisfy $\|\params - \params(0) \|_2 \leq \frac{3(1 + \kappa)\|\target - \pred(0)\|_2}{2\sqrt{\lambda_{\mathrm{min}}(\gram(0))}}$ where $\kappa = \frac{L}{\mu}$
    \begin{equation}
        \|\jacobian(\theta) - \jacobian(0)\|_2 \leq \frac{C}{3} \sqrt{\lambda_{\mathrm{min}}(\gram(0))} .
    \end{equation}
\end{con}
\begin{thm}\label{thm:general-loss}
Under Condition~\ref{con:full-row-rank} and~\ref{con:stable-jaco-2}, but with $\mu$-strongly convex loss function $\ell(\cdot, \cdot)$ with $L$-Lipschitz gradient ($\kappa = \frac{L}{\mu}$), and we set the step size $\lr \leq \frac{2}{\mu + L} \frac{1 - (1 + \kappa)C}{(1+C)^2}$, then we have for $k = 0, 1, 2, ...$
\begin{equation}
    \|\pred(k) - \target \|_2^2 \leq \left(1 - \frac{2\lr\mu L}{\mu + L} \right)^k \| \pred(0) - \target\|_2^2 .
\end{equation}
\end{thm}
The key step of proving Theorem~\ref{thm:general-loss} is to show if $m$ is large enough, then natural gradient descent is approximately gradient descent in the output space. Thus the results can be easily derived according to standard bounds for convex optimization. Due to the page limit, we defer the proof to the Appendix~\ref{app:gen-loss}.
\begin{rem}
    In Theorem~\ref{thm:general-loss}, the convergence rate depends on the condition number $\kappa = \frac{L}{\mu}$, which can be removed if we take into the curvature information of the loss function. In other words, we expect that the bound has no dependency on $\kappa$ if we use the Fisher matrix rather than the classical Gauss-Newton (assuming Euclidean metric in the output space~\citep{luk2018coordinate}) in Theorem~\ref{thm:general-loss}.
\end{rem}

\section{Optimizing Overparameterized Neural Networks}\label{sec:sec-4}
In Section~\ref{sec:ngd}, we analyzed the convergence properties of natural gradient descent, under the abstract Conditions \ref{con:full-row-rank} and \ref{con:stable-jaco}. In this section, we make our analysis concrete by applying it to a specific type of overparameterized network (with a certain random initialization). We show that Conditions 1 and 2 hold with high probability. We therefore establish that NGD exhibits linear convergence to a global minimizer for such networks.

\subsection{Notation}
We let $[m] = \{1, 2, ..., m \}$. We use $\otimes$, $\odot$ to denote the Kronecker and Hadamard products. And we use $\ast$ and $\star$ to denote row-wise and column-wise Khatri-Rao products, respectively. For a matrix $\mathbf{A}$, we use $\mathbf{A}_{ij}$ to denote its $(i, j)$-th entry. We use $\| \cdot \|_2$ to denote the Euclidean norm of a vector or spectral norm of a matrix and $\| \cdot \|_{\fisher}$ to denote the Frobenius norm of a matrix. We use $\lambda_{\mathrm{max}}(\mathbf{A})$ and $\lambda_{\mathrm{min}}(\mathbf{A})$ to denote the largest and smallest eigenvalue of a square matrix, and $\sigma_{\mathrm{max}}(\mathbf{A})$ and $\sigma_{\mathrm{min}}(\mathbf{A})$ to denote the largest and smallest singular value of a (possibly non-square) matrix. For a positive definite matrix $\mathbf{A}$, we use $\kappa_\mathbf{A}$ to denote its condition number, i.e., $\lambda_\mathrm{max}(\mathbf{A})/\lambda_\mathrm{min}(\mathbf{A})$.
We also use $\langle \cdot, \cdot \rangle$ to denote the standard inner product between two vectors. Given an event $E$, we use $\indicator \{E\}$ to denote the indicator function for $E$.

\subsection{Problem Setup}
Formally, we consider a neural network of the following form:
\begin{equation}\label{eq:network}
    f(\weight, \ba, \inputs) = \frac{1}{\sqrt{m}} \sum_{r=1}^m a_r \relu(\weight_r^\top \inputs) ,
\end{equation}
where $\inputs \in \real^d$ is the input, $\weight = \left[\weight_1^\top, ..., \weight_r^\top \right]^\top \in \real^{md}$ is the weight matrix (formed into a vector) of the first layer,
$a_r \in \real$ is the output weight of hidden unit $r$ and $\relu(\cdot)$ is the ReLU activation function (acting entry-wise for vector arguments). For $r \in [m]$, we initialize the weights of first layer $\weight_r \sim \normal(\mathbf{0}, \vars^2 \iden)$ and output weight $a_r \sim \mathbf{unif}\left[\{-1, +1 \} \right]$.

Following~\citet{du2018gradient, wu2019global}, we make the following assumption on the data.
\begin{ass}\label{ass:data}
    For all i, $\|\inputs_i \|_2 = 1$ and $|y_i| = \bigo \left(1 \right)$. For any $i \neq j$, $\inputs_i \nparallel \inputs_j$.
\end{ass}
This very mild condition simply requires the inputs and outputs have standardized norms, and that different input vectors are distinguishable from each other.  Datasets that do not satisfy this condition can be made to do so via simple pre-processing. %

Following~\citet{du2018gradient, oymak2019towards, wu2019global}, we only optimize the weights of the first layer\footnote{We fix the second layer just for simplicity. Based on the same analysis, one can also prove global convergence for jointly training both layers.}, i.e., $\params = \weight$. 
Therefore, natural gradient descent can be simplified to
\begin{equation}
    \weight(k+1) = \weight(k) - \lr \jacobian^\top(\jacobian \jacobian^\top)^{-1} (\pred - \target) .
\end{equation}
Though this is only a shallow fully connected neural network, the objective is still non-smooth and non-convex~\citep{du2018gradient} due to the use of ReLU activation function. We further note that this two-layer network model has been useful in understanding the optimization and generalization of deep neural networks~\citep{xie2016diverse, li2018learning, du2018gradient, arora2019fine, wu2019global}, and some results have been extended to multi-layer networks~\citep{du2018gradient-2}. 

Following~\citet{du2018gradient, wu2019global}, we define the \emph{limiting Gram matrix} as follows: 
\begin{deff}[Limiting Gram Matrix]
The limiting Gram matrix $\gram^\infty \in \real^{n \times n}$ is defined as follows. For $(i, j)$- entry, we have
    \begin{equation}
        \gram_{ij}^\infty = \expect_{\weight \sim \normal(\mathbf{0}, \vars^2 \iden)} \left[\inputs_i^\top \inputs_j \indicator\left\{\weight^\top \inputs_i \geq 0, \weight^\top \inputs_j \geq 0 \right\} \right] = \inputs_i^\top \inputs_j \frac{\pi - \mathrm{arccos}(\inputs_i^\top \inputs_j)}{2\pi} .
    \end{equation}
\end{deff}
This matrix coincides with neural tangent kernel~\citep{jacot2018neural} for ReLU activation function. As shown by~\citet{du2018gradient}, this matrix is positive definite and we define its smallest eigenvalue $\lambda_0 \triangleq \lambda_{\mathrm{min}}(\gram^\infty) > 0$.
In the same way, we can define its finite version $\gram(t) = \jacobian(t) \jacobian(t)^\top$ with $(i,j)$-entry $\gram_{ij}(t) = \frac{1}{m}  \inputs_i^\top \inputs_j \sum_{r\in[m]} \indicator \left\{ \weight_r(t)^\top\inputs_i\geq 0, \weight_r(t)^\top \inputs_j \geq 0 \right\}$.

\begin{figure}[t]
    \centering
    \vspace{-0.3cm}
    \includegraphics[width=0.95\textwidth]{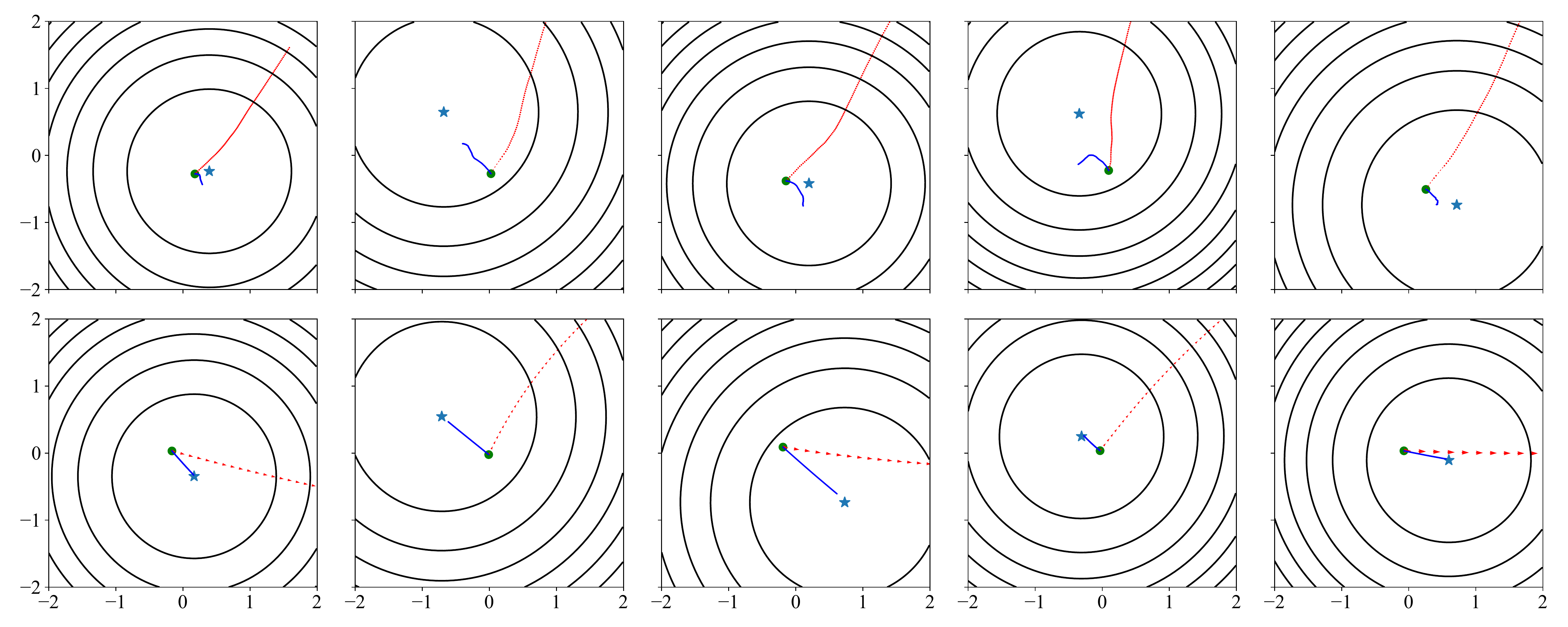}
    \vspace{-0.2cm}
    \caption{Visualization of {\color[rgb]{0,0,0.9} natural gradient} update and {\color[rgb]{0.9,0,0}gradient descent} update in the output space (for a randomly initialized network). We take two classes (4 and 9) from MNIST~\citep{lecun1998gradient} and generate the targets (denoted as star in the figure) by $f(x) = x - 0.5 + 0.3 \times \normal(0, \iden)$ where $x \in \real^2$ is one-hot target. We get natural gradient update by running 100 iterations of conjugate gradient~\citep{martens2010deep}. \textbf{The first row}: a MLP with two hidden layers and 100 hidden units in each layer. \textbf{The second row}: a MLP with two hidden layers and 6000 hidden units in each layer. In both cases, ReLU activation function was used. We interpolate the step size from 0 to 1. For the over-parameterized network (in the second row), natural gradient descent (implemented by conjugate gradient) matches output space gradient well.}
    \label{fig:true_label} 
    \vspace{-0.3cm}
\end{figure}
\subsection{Exact Natural Gradient Descent}
In this subsection, we present our result for this setting. The main difficulty is to show that Conditions~\ref{con:full-row-rank} and~\ref{con:stable-jaco} hold. Here we state our main result.
\begin{thm}[Natural Gradient Descent for overparameterized Networks]\label{thm:d-convergence}
Under Assumption~\ref{ass:data}, if we i.i.d initialize $\weight_r \sim \normal(\mathbf{0}, \vars^2\mathbf{I})$, $a_r \sim \mathrm{unif}[\{-1,+1 \}]$ for $r \in [m]$, we set the number of hidden nodes $m = \Omega \left( \frac{n^4}{\vars^2\lambda_0^4 \delta^3} \right)$, and the step size $\lr = \bigo(1)$, then with probability at least $1 - \delta$ over the random initialization we have for $k = 0, 1, 2, ...$
\begin{equation}
    \|\pred(k) - \target\|_2^2 \leq \left(1 - \lr\right)^k \|\pred(0) - \target\|_2^2 .
\end{equation}
\end{thm}
Even though the objective is non-convex and non-smooth, natural gradient descent with a constant step size enjoys a linear convergence rate. For large enough $m$, we show that the learning rate can be chosen up to $1$, so NGD can provably converge within $\bigo\left(1\right)$ steps. Compared to analogous bounds for gradient descent~\citep{du2018gradient-2, oymak2019towards, wu2019global}, we improve the maximum allowable learning rate from $\bigo(1/n)$ to $\bigo(1)$ and also get rid of the dependency on $\lambda_0$. Overall, NGD (Theorem~\ref{thm:d-convergence}) gives an $\bigo(\lambda_0/n)$ improvement over gradient descent.

Our strategy to prove this result will be to show that for the given choice of random initialization, Condition~\ref{con:full-row-rank} and~\ref{con:stable-jaco} hold with high probability. For proving Condition~\ref{con:full-row-rank} hold, we used matrix concentration inequalities. For Condition~\ref{con:stable-jaco}, we show that $\|\jacobian - \jacobian(0) \|_2 = \bigo\left(m^{-1/6} \right)$, which implies the Jacobian is stable for wide networks. For detailed proof, we refer the reader to the Appendix~\ref{app:d-convergence}.

\subsection{Approximate Natural Gradient Descent with K-FAC}
Exact natural gradient descent is quite expensive in terms of computation or memory. In training deep neural networks, K-FAC~\citep{martens2015optimizing} has been a powerful optimizer for leveraging curvature information while retaining tractable computation. The K-FAC update rule for the two-layer ReLU network is given by
\begin{equation}
    \weight(k+1) = \weight(k) - \lr \underbrace{\left[(\bX^\top\bX)^{-1} \otimes (\bS(k)^\top\bS(k))^{-1} \right]}_{\fisher_\mathrm{K-FAC}^{-1}} \jacobian(k)^\top (\pred(k) - \target) .
\end{equation}
where $\bX \in \real^{n \times d}$ denotes the matrix formed from the $n$ input vectors (i.e. $\bX = [\inputs_1, ..., \inputs_n]^\top$), and $\bS = [\phi^\prime(\bX \weight_1), ..., \phi^\prime(\bX \weight_m)] \in \real^{n \times m}$ is the matrix of pre-activation derivatives. 
Under the same argument as the Gram matrix $\gram^\infty$, we get that ${\bS^\infty\bS^\infty}^\top$ is strictly positive definite with smallest eigenvalue $\lambda_{\bS}$ (see Appendix~\ref{app:pd-pre-act} for detailed proof).

We show that for sufficiently wide networks, K-FAC does converge linearly to a global minimizer. We further show, with a particular transformation on the input data, K-FAC does match the optimization performance of exact natural gradient for two-layer ReLU networks. Here we state the main result. 
\begin{thm}[K-FAC]\label{thm:kfac}
    Under the same assumptions as in Theorem~\ref{thm:d-convergence}, plus the additional assumption that $\mathrm{rank}(\bX) = d$, if we set the number of hidden units $m = \bigo\left(\frac{n^4}{\vars^2\lambda_{\bS}^4 \kappa_{\bX^\top\bX}^4 \delta^3} \right)$ and step size $\lr = \bigo \left(\lambda_\mathrm{min}\left(\bX^\top \bX\right) \right)$, then with probability at least $1 - \delta$ over the random initialization, we have for $k = 0, 1, 2, ...$
    \begin{equation}
        \|\pred(k) - \target\|_2^2 \leq \left(1 - \frac{\lr}{\lambda_\mathrm{max}(\bX^\top\bX) }\right)^k \|\pred(0) - \target\|_2^2 .
    \end{equation}
\end{thm}
The key step in proving Theorem~\ref{thm:kfac} is to show 
\begin{equation}
    \pred(k+1) - \pred(k) \approx \left[\left(\bX (\bX^\top \bX)^{-1} \bX^\top \right) \odot \iden \right] \left( \target - \pred(k) \right) .
\end{equation}
\begin{rem}
    The convergence rate of K-FAC is captured by the condition number of the matrix $\bX^\top \bX$, as opposed to gradient descent~\citep{du2018gradient, oymak2019towards}, for which the convergence rate is determined by the condition number of the Gram matrix $\gram$. 
\end{rem}
\begin{rem}
The dependence of the convergence rate on $\kappa_{\bX^\top \bX}$ in Theorem \ref{thm:kfac} may seem paradoxical, as K-FAC is invariant to invertible linear transformations of the data (including those that would change $\kappa_{\bX^\top \bX}$). But we note that said transformations would also make the norms of the input vectors non-uniform, thus violating Assumption~\ref{ass:data} in a way that isn't repairable. Interestingly, there exists an invertible linear transformation which, if applied to the input vectors and followed by normalization, produces vectors that simultaneously satisfy Assumption~\ref{ass:data} and the condition $\kappa_{\bX^\top \bX} = 1$ (thus improving the bound in Theorem \ref{thm:kfac} substantially). See Appendix~\ref{app:forster} for details. Notably, K-FAC is \emph{not} invariant to such pre-processing, as the normalization step is a nonlinear operation.
\end{rem}
To quantify the degree of overparameterization (which is a function of the network width $m$) required to achieve global convergence under our analysis, we must estimate $\lambda_\bS$. To this end, we observe that $\gram = \bX\bX^\top \odot \bS \bS^\top$, and then apply the following lemma:
\begin{lem}\label{lem:hadamard}[\citet{schur1911bemerkungen}]
    For two positive definite matrices $\mathbf{A}$ and $\mathbf{B}$, we have
    \begin{equation}
    \begin{aligned}
        &\lambda_{\mathrm{max}}\left(\mathbf{A} \odot \mathbf{B} \right) \leq \max_{i} \mathbf{A}_{ii} \lambda_\mathrm{max}(\mathbf{B}) \\
        &\lambda_{\mathrm{min}}\left(\mathbf{A} \odot \mathbf{B} \right) \geq \min_{i} \mathbf{A}_{ii} \lambda_\mathrm{min}(\mathbf{B}) 
    \end{aligned}
    \end{equation}
\end{lem}
The diagonal entries of $\bX\bX^\top$ are all $1$ since the inputs are normalized. Therefore, we have $\lambda_0 \geq \lambda_\bS$ according to Lemma~\ref{lem:hadamard}, and hence K-FAC requires a slightly higher degree of overparameterization than exact NGD under our analysis.

\subsection{Bounding $\lambda_0$}
As pointed out by~\citet{allen2018convergence}, it is unclear if $1/\lambda_0$ is small or even polynomial. Here, we bound $\lambda$ using matrix concentration inequalities and harmonic analysis. To leverage harmonic analysis, we have to assume the data $\inputs_i$ are drawn i.i.d.~from the unit sphere\footnote{This assumption is not too stringent since the inputs are already normalized. Moreover, we can relax the assumption of unit sphere input to separable input, which is used in~\citet{li2018learning, allen2018convergence, zou2018stochastic}. See~\citet{oymak2019towards} (Theorem I.1) for more details.}.
\begin{thm}\label{thm:bound-lambda}
Under this assumption on the training data, with probability $1 - n \exp (-n^\beta/4)$,
\begin{equation}
    \lambda_0 \triangleq \lambda_{\mathrm{min}}(\gram^\infty) \geq n^\beta/2, \; \text{where} \; \beta \in (0, 0.5)
\end{equation}
\end{thm}
Basically, Theorem~\ref{thm:bound-lambda} says that the Gram matrix $\gram^\infty$ should have high chance of having large smallest eigenvalue if the training data are uniformly distributed. Intuitively, we would expect the smallest eigenvalue to be very small if all $\inputs_i$ are similar to each other. Therefore, some notion of diversity of the training inputs is needed. We conjecture that the smallest eigenvalue would still be large if the data are $\delta$-separable (i.e., $\left\|\inputs_i - \inputs_j \right\|_2 \geq \delta$ for any pair $i, j \in [n]$), an assumption adopted by~\citet{li2018learning, allen2018convergence, zou2018stochastic}.

\section{Generalization analysis}\label{sec:sec-5}
It is often speculated that NGD or other preconditioned gradient descent methods (e.g., Adam) perform worse than gradient descent in terms of generalization~\citep{wilson2017marginal}. In this section, we show that NGD achieves the same generalization bounds which have been proved for GD, at least for two-layer ReLU networks.

Consider a loss function $\ell: \real \times \real \rightarrow \real$. The expected risk over the data distribution $\data$ and the empirical risk over a training set $\tdata = \left\{ (\inputs_i, y_i) \right\}_{i=1}^n$ are defined as
\begin{equation}
    \loss_\data(f) = \expect_{(\inputs, y) \sim \data}\left[ \ell(f(\inputs), y) \right] \;\; \text{and} \;\; \loss_\tdata(f) = \frac{1}{n}\sum_{i=1}^n \ell(f(\inputs_i), y_i)
\end{equation}
It has been shown~\citep{neyshabur2018the} that the Redemacher complexity~\citep{bartlett2002rademacher} for two-layer ReLU networks depends on $\left\|\weight - \weight(0) \right\|_2$.
By the standard Rademacher complexity generalization bound, we have the following bound (see Appendix~\ref{app:thm-generalization} for proof):
\begin{thm}\label{thm:generalization}
Given a target error parameter $\epsilon > 0$ and failure probability $\delta \in (0, 1)$. Suppose $\vars = \bigo \left( \epsilon \sqrt{\lambda_0 \delta}\right)$ and $m \geq \vars^{-2}\mathrm{poly}\left(n, \lambda_0^{-1}, \delta^{-1}, \epsilon^{-1} \right)$. For any 1-Lipschitz loss function, with probability at least $1 - \delta$ over random initialization and training samples, the two-layer neural network $f(\weight, \ba)$ trained by NGD for $k \geq \Omega\left( \frac{1}{\lr}\log \frac{1}{\epsilon \delta} \right)$ iterations has expected loss $\loss_\data(f(\weight, \ba)) = \expect_{(\inputs, y) \sim \data} \left[\ell(f(\weight, \ba, \inputs), y) \right]$ bounded as:
\begin{equation}
    \loss_\data(f(\weight, \ba)) \leq \sqrt{\frac{2\target^\top (\gram^\infty)^{-1}\target}{n}} + 3\sqrt{\frac{\log(6/\delta)}{2n}} + \epsilon
\end{equation}
\end{thm}
which matches the bound for gradient descent in~\citet{arora2019fine}. For detailed proof, we refer the reader to the Appendix~\ref{app:thm-generalization}.

\section{Conclusion}
We've analyzed for the first time the rate of convergence to a global optimum for (both exact and approximate) natural gradient descent on nonlinear neural networks. Particularly, we identified two conditions which guarantee the global convergence, i.e., the Jacobian matrix with respect to the parameters has full row rank and stable for perturbations around the initialization. Based on these insights, we improved the convergence rate of gradient descent by a factor of $\bigo(\lambda_0 / n)$ on two-layer ReLU networks by using natural gradient descent. Beyond that, we also showed that the improved convergence rates don't come at the expense of worse generalization.

\section*{Acknowledgements}
We thank Jeffrey Z. HaoChen, Shengyang Sun and Mufan Li for helpful discussion.

\bibliography{neurips_2019.bib}
\bibliographystyle{plainnat}

\appendix
\newpage

\input{appendix}

\end{document}

%% file: appendix.tex
\section{The Forster Transform}
\label{app:forster}

In a breakthrough paper in the area of communication complexity, \citet{forster2002linear} used the existence of a certain kind of dataset transformation as the key technical tool in the proof of his main result.  The Theorem which establishes the existence of this transformation is paraphrased below.

\begin{thm}[\citet{forster2002linear}, Theorem 4.1]
\label{thm:forster}
Suppose $\bX \in \real^{n \times d}$ is a matrix such that all subsets of size at most $d$ of its rows are linearly independent. Then there exists an invertible matrix $\bA \in \real^{d \times d}$ such that if we post-multiply $\bX$ by $\bA$ (i.e. apply $\bA$ to each row), and then normalize each row by its 2-norm, the resulting matrix $\bZ \in \real^{n \times d}$ satisfies $\bZ^\top \bZ = \frac{n}{d} \iden_d$.
\end{thm}
\begin{rem}
Note that the technical condition about linear independence can be easily be made to hold for an arbitrary $\bX$ by adding an infinitesimal random perturbation, assuming it doesn't hold to begin with.
\end{rem}

This result basically says that for any set of vectors, there is a linear transformation of said vectors which makes their normalized versions (given by the rows of $\bZ$) satisfy $\bZ^\top \bZ = \frac{n}{d} \iden_d$. So by combining this linear transformation with normalization we produce a set of vectors that simultaneously satisfy Assumption~\ref{ass:data}, while also satisfying $\kappa_{\bZ^\top \bZ} = 1$.

Forster's proof of Theorem \ref{thm:forster} can be interpreted as defining a transformation function on $\bZ$ (initialized at $\bX$), and showing that it has a fixed point with the required properties. One can derive an algorithm from this by repeatedly applying the transformation to $\bZ$, which consists of "whitening" followed by normalization, until $\bZ^\top \bZ$ is sufficiently close to $\frac{n}{d} \iden_d$.  The $\bA$ matrix is then simply the product of the "whitening" transformation matrices, up to a scalar constant. While no explicit finite-time convergence guarantees are given for this algorithm by~\citet{forster2002linear}, we have implemented it and verified that it does indeed converge at a reasonable rate. The algorithm is outlined below.

\begin{algorithm}[htb]
  \caption{Forster Transform}
  \begin{algorithmic}[1]
    \STATE \textbf{INPUT:} Matrix $\bX \in \real^{n \times d}$ satisfying the hypotheses of Theorem \ref{thm:forster}
    \STATE $\bZ \gets \bX$
    \STATE $\bA \gets \iden_d$
    \WHILE{error tolerance exceeded}
    \STATE $\bT \gets (\bZ^\top \bZ)^{-\frac{1}{2}}$
    \STATE $\bZ \gets \bZ \bT$
    \STATE $\bZ \gets \mbox{normalize-rows}(\bZ)$
    \STATE $\bA \gets \bA \bT$
    \STATE $\bA \gets \frac{1}{[\bA]_{1,1}} \bA$
    \ENDWHILE
  \STATE \textbf{OUTPUT:} $\bA \in \real^{d \times d}$ with the properties stated in Theorem \ref{thm:forster} (up to an error tolerance)
  \end{algorithmic}
  \label{alg:learning}
\end{algorithm}

\section{Proof of Theorem~\ref{thm:ngd}}
\label{app:main_proof}

We prove the result in two steps: we first provide a convergence analysis for natural gradient flow, i.e., natural gradient descent with infinitesimal step size, and then take into account the error introduced by discretization and show global convergence for natural gradient descent.
    
To guarantee global convergence for natural gradient flow, we only need to show that the Gram matrix is positive definite throughout the training. Intuitively, for successfully finding global minima, the network must satisfy the following condition, i.e., the gradient with respect to the parameters $\nabla_{\params} \loss(\params)$ is zero only if the gradient in the output space $\nabla_{\pred}\loss(\params) = \mathbf{0}$ is zero. It suffices to show that the Gram matrix is positive definite, or equivalently, the Jacobian matrix is full row rank. 

By Condition~\ref{con:full-row-rank} and Condition~\ref{con:stable-jaco}, we immediately obtain the following lemma that if the parameters stay close to the initialization, then the Gram matrix is positive definite throughout the training.
\begin{lem}
\label{lem:g_bound}
If $\|\params - \params(0) \|_2 \leq \frac{3\|\target - \pred(0)\|_2}{\sqrt{\lambda_{\mathrm{min}}(\gram(0))}}$, then we have $\lambda_\mathrm{min}(\gram) \geq \frac{4}{9} \lambda_\mathrm{min}(\gram(0))$.
\end{lem}
\begin{proof}[Proof of Lemma~\ref{lem:g_bound}]
Based on the inequality that $\sigma_\mathrm{min}(\bA + \bB) \geq \sigma_\mathrm{min}(\bA) - \sigma_\mathrm{max}(\bB)$ where $\sigma$ denotes singular value, we have
\begin{equation}
    \sigma_\mathrm{min}(\jacobian) \geq \sigma_\mathrm{min}(\jacobian(0)) - \|\jacobian - \jacobian(0)\|_2
\end{equation}
By Condition~\ref{con:stable-jaco}, we have $\|\jacobian - \jacobian(0)\|_2 \leq \frac{1}{3} \sqrt{\lambda_\mathrm{min}(\gram(0))}$, thus we get $\sigma_\mathrm{min}(\jacobian) \geq \frac{2}{3}\sqrt{\lambda_\mathrm{min}(\gram(0))}$ which completes the proof.
\end{proof}
With the assumption that $\|\params - \params(0) \|_2 \leq \frac{3\|\target - \pred(0)\|_2}{\sqrt{\lambda_{\mathrm{min}}(\gram(0))}}$ throughout the training, we are now ready to prove global convergence for natural gradient flow. Recall the dynamics of natural gradient flow in weight space,
\begin{equation}
    \frac{d }{d t}\params(t) = \frac{1}{n} \fisher(t)^\dagger \jacobian(t)^\top (\target - \pred(t))
\end{equation}
Accordingly, we can calculate the dynamics of the network predictions.
\begin{equation}\label{eq:c-pred-dyn}
\begin{aligned}
    \frac{d}{dt}\pred(t) &= \frac{1}{n} \jacobian(t)\fisher(t)^\dagger \jacobian(t)^\top (\target - \pred(t)) \\
        & = \jacobian(t) \jacobian(t)^\top \gram(t)^{-1} \gram(t)^{-1}\jacobian(t)\jacobian(t)^\top (\target - \pred(t)) \\
\end{aligned}
\end{equation}
Since the Gram matrix $\gram(t)$ is positive definite, its inverse does exist. Therefore, we have 
\begin{equation}
    \frac{d}{dt}\pred(t) = \target - \pred(t)
\end{equation}
By the chain rule, we get the dynamics of the loss in the following form:
\begin{equation}\label{eq:loss-dyn}
    \frac{d}{dt}\left\| \target - \pred(t) \right\|_2^2 = -2 (\target - \pred(t))^\top (\target - \pred(t))
\end{equation}
By integrating eqn.~\eqref{eq:loss-dyn}, we find that $\left\| \target - \pred(t) \right\|_2^2 = \exp(-2t) \left\| \target - \pred(0) \right\|_2^2$. 

That completes the continuous time analysis, under the assumption that the parameters stay close to the initialization. The discrete case follows similarly, except that we need to account for the discretization error.
Analogously to eqn.~\eqref{eq:c-pred-dyn}, we calculate the difference of predictions between two consecutive iterations.
    \begin{equation}\label{eq:update}
    \begin{aligned}
        \pred(k+1) - \pred(k) &= \pred\left(\params(k) - \lr \jacobian(k)^\top \gram(k)^{-1} (\pred(k) - \target)\right) - \pred(\params(k)) \\
        &= -\int_{s=0}^1 \left\langle \frac{\partial \pred\left(\params(s) \right)}{\partial \params^\top}, \: \lr\jacobian(k)^\top \gram(k)^{-1}(\pred(k) - \target) \right\rangle ds \\
        &= \underbrace{-\int_{s=0}^1 \left\langle \frac{\partial \pred\left(\params(k)\right)}{\partial \params^\top}, \: \lr\jacobian(k)^\top \gram(k)^{-1}(\pred(k) - \target) \right\rangle ds}_{\lr \left(\target - \pred(k)\right)} \\
        & \quad + \underbrace{\int_{s=0}^1 \left\langle \frac{\partial \pred\left(\params(k)\right)}{\partial \weight^\top} - \frac{\partial \pred\left(\params(s)\right)}{\partial \params^\top}, \: \lr\jacobian(k)^\top \gram(k)^{-1}(\pred(k) - \target) \right\rangle ds}_{\circled{1}} ,
    \end{aligned}
    \end{equation}
where we have defined $\params(s) = s\params(k+1) + (1-s)\params(k) = \params(k) - s \lr \jacobian(k)^\top \gram(k)^{-1} (\pred(k) - \target)$. 

Next we bound the norm of the second term ($\circled{1}$) in the RHS of eqn.~\eqref{eq:update}. Using Condition~\ref{con:stable-jaco} and Lemma \ref{lem:g_bound} we have that
    \begin{equation}
    \begin{aligned}
        \left\|\circled{1} \right\|_2 & \leq \lr \left\|\int_{s=0}^1 \jacobian(\params(s)) - \jacobian(\params(k)) \, ds \right\|_2 \left\|\jacobian(k)^\top \gram(k)^{-1} (\pred(k) - \target)\right\|_2 \\
        & \leq \lr \frac{2C}{3} \sqrt{\lambda_\mathrm{min}(\gram(0))} \frac{1}{\sqrt{\lambda_{\mathrm{min}}(\gram(k))}} \left\|\pred(k) - \target \right\|_2 \\
        & \leq \lr \frac{2C}{3} \sqrt{\lambda_\mathrm{min}(\gram(0))} \frac{3}{2\sqrt{\lambda_{\mathrm{min}}(\gram(0))}} \left\|\pred(k) - \target \right\|_2 \\
        & = \lr C \left\|\pred(k) - \target \right\|_2 .
    \end{aligned}
    \end{equation}
In the first inequality, we used the fact (based on Condition~\ref{con:stable-jaco}) that
\begin{equation}
\begin{aligned}
	\left\|\int_{s=0}^1 \jacobian(\params(s)) - \jacobian(\params(k)) \, ds \right\|_2 & \leq \left\|\jacobian(\theta(k)) - \jacobian(\theta(0)) \right\|_2 + \left\|\jacobian(\theta(k+1)) - \jacobian(\theta(0)) \right\|_2 \\
	& \leq \frac{2C}{3} \sqrt{\lambda_\mathrm{min}(\gram(0))}
\end{aligned}
\end{equation}
Lastly, we have
\begin{equation}\label{eq:norm-decomp}
\begin{aligned}
\|\target - \pred(k+1)\|_2^2 &= \|\target - \pred(k) - (\pred(k+1) - \pred(k))\|_2^2 \\
&= \|\target - \pred(k)\|_2^2 - 2 (\target - \pred(k))^\top(\pred(k+1) - \pred(k)) + \|\pred(k+1) - \pred(k)\|_2^2 \\
&\leq \left(1 - 2\lr + 2\lr C + \lr^2 \left( 1+C\right)^2\right) \|\target - \pred(k)\|_2^2 \\
&\leq \left(1 - \lr\right)\|\target - \pred(k)\|_2^2 .
\end{aligned}
\end{equation}
In the last inequality of eqn.~\eqref{eq:norm-decomp}, we use the assumption that $\lr \leq \frac{1 - 2C}{(1+C)^2}$.

So far, we have assumed the parameters fall within a certain radius around the initialization. We now justify this assumption.
\begin{lem}\label{lem:restricted-class}
If Conditions~\ref{con:full-row-rank} and \ref{con:stable-jaco} hold, then as long as $\lambda_\mathrm{min}(\gram(k)) \geq \frac{4}{9} \lambda_\mathrm{min}(\gram(0))$, we have
\begin{equation}\label{eq:weight-dist-con}
    \|\params(k+1) - \params(0) \|_2 \leq \frac{3\|\target - \pred(0) \|_2}{\sqrt{\lambda_\mathrm{min}(\gram(0))}} .
\end{equation}
\end{lem}
\begin{proof}[Proof of Lemma~\ref{lem:restricted-class}]
We use the norm of each update to bound the distance of the parameters to the initialization.
\begin{equation}
\begin{aligned}
    \|\params(k+1) - \params(0) \|_2 & \leq \lr\sum_{s=0}^k \|\jacobian(s)^\top \gram(s)^\top (\target - \pred(s)) \|_2 \\
    & \leq  \lr\sum_{s=0}^k \frac{\|\target - \pred(s) \|_2}{\sqrt{\lambda_\mathrm{min}(\gram(s))}} \\
    & \leq  \lr\sum_{s=0}^k \frac{(1 - \lr)^{s/2}\|\target - \pred(0) \|_2}{\sqrt{\frac{4}{9}\lambda_\mathrm{min}(\gram(0))}} \\
    & \leq \frac{3\|\target - \pred(0) \|_2}{\sqrt{\lambda_\mathrm{min}(\gram(0))}} .
\end{aligned}
\end{equation}
This completes the proof.
\end{proof}
At first glance, the proofs in Lemma~\ref{lem:g_bound} and~\ref{lem:restricted-class} seem to be circular. Here, we prove that their assumptions continue to be jointly satisfied. 
\begin{lem}\label{lem:contradiction}
	Assuming Conditions~\ref{con:full-row-rank} and~\ref{con:stable-jaco}, we have (1) $\|\params(k) - \params(0) \|_2 \leq \frac{3\|\target - \pred(0) \|_2}{\sqrt{\lambda_\mathrm{min}(\gram(0))}}$ and (2) $\lambda_\mathrm{min}(\gram(k)) \geq \frac{4}{9} \lambda_\mathrm{min}(\gram(0))$ throughout the training.
\end{lem}
\begin{proof}%
	We prove the lemma by contradiction. Suppose the conclusion does not hold for all iterations. Let's say (1) holds at iteration $k = 0, ..., k_0$ but not iteration $k_0 + 1$. Then we know, there must exist $0 < k^\prime \leq k_0$ such that from $\lambda_\mathrm{min}(\gram(k^\prime)) < \frac{4}{9} \lambda_\mathrm{min}(\gram(0))$, otherwise we can show that (1) holds at iteration $k_0 + 1$ as well by Lemma~\ref{lem:restricted-class}. However, by Lemma~\ref{lem:g_bound}, we know that $\lambda_\mathrm{min}(\gram(k^\prime)) \geq \frac{4}{9} \lambda_\mathrm{min}(\gram(0))$ since (1) holds for $k = 0, ..., k_0$, contradiction.
\end{proof}
Notably, Lemma~\ref{lem:contradiction} shows that $\|\params(k) - \params(0) \|_2 \leq \frac{3\|\target - \pred(0) \|_2}{\sqrt{\lambda_\mathrm{min}(\gram(0))}}$ and $\lambda_\mathrm{min}(\gram(k)) \geq \frac{4}{9} \lambda_\mathrm{min}(\gram(0))$ throughout the training if Conditions~\ref{con:full-row-rank} and~\ref{con:stable-jaco} hold. This completes the proof of our main result.

\section{Proof of Theorem~\ref{thm:general-loss}}\label{app:gen-loss}
Here, we prove Theorem~\ref{thm:general-loss} by induction. Our inductive hypothesis is the following condition.
\begin{con}
At the $k$-th iteration, we have $\| \target - \pred(k+1) \|_2^2 \leq \left(1 - \frac{2\lr\mu L}{\mu + L} \right) \| \target - \pred(k) \|_2^2$.
\end{con}
We first use the norm of the gradient to bound the distance of the weights. Here we slightly abuse notation, $\loss(\pred) = \sum_{i=1}^n \ell(u_i, y_i)$.
\begin{equation}
\begin{aligned}
    \|\params(k+1) - \params(0)\|_2 &\leq \lr \sum_{s=0}^k \|\jacobian(k)^\top \gram(k)^{-1} \nabla_\pred \loss(\pred(k))\|_2 \\
    &\leq \lr L \sum_{s=0}^k \|\jacobian(k)^\top \gram(k)^{-1}\|_2 \|\target - \pred(k)\|_2 \\
    &\leq \lr L \sum_{s=0}^\infty \left(1 - \frac{2\lr\mu L}{\mu + L} \right)^{s/2}\frac{\|\target - \pred(0)\|_2}{\sqrt{\frac{4}{9}\lambda_\mathrm{min}(\gram(0))}} \\
    &= \frac{3(1 + \kappa)\|\target - \pred(0)\|_2}{2\sqrt{\lambda_\mathrm{min}(\gram(0))}}
\end{aligned}
\end{equation}
where $\kappa = \frac{L}{\mu}$. The second inequality is based on the $L$-Lipschitz gradient assumption\footnote{That the gradient of $\ell$ is $L$-Lipschitz implies the gradient of $\loss$ is also $L$-Lipschitz.} and the fact that $\nabla_\pred \loss(\target) = 0$. 
Also, we have
\begin{equation}
\begin{aligned}
    \pred(k+1) - \pred(k) = \lr \nabla_{\pred} \loss(\pred(k)) + \lr \mathbf{P}(k) \nabla_\pred \loss(\pred(k)) 
\end{aligned}
\end{equation}
In analogy to eqn.~\eqref{eq:update}, $\mathbf{P}(k) = \int_{s=0}^1  \left(\jacobian(k) - \jacobian(\params(s))\right) \jacobian(k)^\top \gram(k)^{-1} ds $. Next, we introduce a well-known Lemma for $\mu$-strongly convex and $L$-Lipschitz gradient loss.
\begin{lem}[Co-coercivity for $\mu$-strongly convex loss]\label{lem:coercivity}
    If the loss function is $\mu$-strongly convex with $L$-Lipschitz gradient, for any $\pred, \target \in \real^n$, the following inequality holds.
\begin{equation}
    \left(\nabla\loss(\pred) - \nabla\loss(\target) \right)^\top(\pred - \target) \geq \frac{\mu L}{\mu + L} \|\pred - \target \|_2^2 + \frac{1}{\mu + L} \|\nabla \loss(\pred) - \nabla \loss(\target) \|_2^2
\end{equation}
\end{lem}

Now, we are ready to bound $\left\|\target - \pred(k+1) \right\|_2^2$:
\begin{equation}
\begin{aligned}
    \left\|\pred(k+1) - \target \right\|_2^2 &= \left\| \pred(k) - \lr(\iden+\mathbf{P}(k)) \nabla_\pred \loss(\pred(k)) - \target  \right\|_2^2 \\
    &\leq \left\| \pred(k) - \target \right\|_2^2 - 2\lr \nabla_\pred \loss(\pred(k))^\top \left(\pred(k) - \target \right) \\
    & \quad + 2\lr \|\mathbf{P} \|_2 \|\nabla_\pred \loss(\pred(k)) \|_2 \|\pred(k) - \target \|_2 + \lr^2 (1+\|\mathbf{P}(k)\|_2)^2 \left\| \nabla_\pred \loss(\pred(k)) \right\|_2^2 \\
    & \leq \left\| \pred(k) - \target \right\|_2^2 - 2\lr \frac{\mu L}{\mu + L}\left\| \pred(k) - \target \right\|_2^2 - \frac{2\lr}{\mu + L} \|\nabla_\pred \loss(\pred(k))\|_2^2 \\
    & \quad + 2\lr \|\mathbf{P} \|_2 \|\nabla_\pred \loss(\pred(k)) \|_2 \|\pred(k) - \target \|_2
    + \lr^2 (1+\|\mathbf{P}(k)\|_2)^2 \left\| \nabla_\pred \loss(\pred(k)) \right\|_2^2 \\
    & \leq \left(1 - \frac{2\lr \mu L}{\mu + L} \right)\left\| \pred(k) - \target \right\|_2^2 + \lr^2 (1 + C)^2 \|\nabla_\pred \loss(\pred(k)) \|_2^2\\
    & \quad + 2\lr C \|\nabla_\pred \loss(\pred(k)) \|_2 \|\pred(k) - \target \|_2 - \frac{2\lr}{\mu + L}\|\nabla_\pred \loss(\pred(k)) \|_2^2 \\
    & \leq \left(1 - \frac{2\lr \mu L}{\mu + L} \right)\left\| \pred(k) - \target \right\|_2^2
\end{aligned}
\end{equation}
For the second inequality we used Lemma~\ref{lem:coercivity} and the fact that $\nabla_\pred \loss(\target) = \mathbf{0}$. For the third inequality we used the result of Condition~\ref{con:stable-jaco-2} that $\|\mathbf{P}(\pred(k)) \|_2 \leq C$.
For the last inequality we used the fact that the loss is $\mu$-strongly convex and $\lr < \frac{2}{\mu + L} \frac{1 - (1 + \kappa)C}{(1+C)^2}$.

\begin{proof}[Proof of Lemma~\ref{lem:coercivity}]
    For convex function $f$ with Lipschtiz gradient, we have the following co-coercivity property~\citep{boyd2004convex}:
    \begin{equation}
        \left( \nabla f(\inputs) - \nabla f(\target) \right)^\top (\inputs - \target) \geq \frac{1}{L} \left\|\nabla f(\inputs) - \nabla f(\target) \right\|_2^2
    \end{equation}
    Then, for $\mu$-strongly convex loss, we can define $g(\inputs) = f(\inputs) - \frac{\mu}{2}\|\inputs\|_2^2$, which is a convex function with $L-\mu$ Lipschitz gradient. By co-coercivity of $g$, we get
    \begin{equation}
        \left( \nabla g(\inputs) - \nabla g(\target) \right)^\top (\inputs - \target) \geq \frac{1}{L-\mu} \left\|\nabla g(\inputs) - \nabla g(\target) \right\|_2^2
    \end{equation}
    After plugging $\nabla g(\inputs) = \nabla f(\inputs) - \mu \inputs $ and some manipulations, we have
    \begin{equation}
        \left(\nabla f(\inputs) - \nabla f(\target) \right)^\top(\inputs - \target) \geq \frac{\mu L}{\mu + L} \|\inputs - \target \|_2^2 + \frac{1}{\mu + L} \|\nabla f(\inputs) - \nabla f(\target) \|_2^2
    \end{equation}
    This completes the proof.
\end{proof}

\section{Proofs for Section~\ref{sec:sec-4}}
\subsection{Proof of Theorem~\ref{thm:d-convergence}}\label{app:d-convergence}
Our strategy to prove this result will be to show that for the given choice of random initialization, Conditions~\ref{con:full-row-rank} and~\ref{con:stable-jaco} hold with high probability.

We start with Condition~\ref{con:full-row-rank}, which requires that the inital Gram matrix $\gram(0)$ is non-singular, or equivalently that $\jacobian(0)$ has full row rank. Because $\gram(0)$ is a sum of random matrices, where the expectation of each random matrix is $\frac{1}{m} \gram^\infty$, we can bound its smallest eigenvalue using matrix concentration inequalities. Doing so gives us the following lemma.
\begin{lem}\label{lem:full-rank}
If $m = \Omega\left(\frac{n \log(n/\delta)}{\lambda_0} \right)$, then we have with probability at least $1-\delta$ that $\lambda_{\mathrm{min}}(\gram(0)) \geq \frac{3}{4}\lambda_0$ (which implies that $\jacobian(0)$ has full row rank).
\end{lem}

Next we introduce another lemma, which we will use to show that Condition~\ref{con:stable-jaco} holds with high probability. 
\begin{lem}\label{lem:jaco-pert}
 With probability at least $1 - \delta$, for all weight vectors $\weight$ that satisfy $\left\|\weight - \weight(0) \right\|_2 \leq R$, we have the following bound:
    \begin{equation}\label{eq:jacobian-bound}
        \|\jacobian - \jacobian(0) \|_2^2 \leq \frac{2nR^{2/3}}{\vars^{2/3}\delta^{2/3}m^{1/3}} .
    \end{equation}
\end{lem}
Notably, Lemma~\ref{lem:jaco-pert} says that as long as the weights are close to the random initialization, the corresponding Jacobian matrix also stays close to the Jacobian matrix of inital weights.
Therefore, we might expect that if the distance to the initialization is sufficiently small, then Condition~\ref{con:stable-jaco} would hold with high probability.

By taking $R = \frac{3\|\target - \pred(0) \|_2}{\sqrt{\lambda_\mathrm{min}(\gram(0)}}$ in eqn.~\eqref{eq:jacobian-bound}, we have $\left\| \jacobian - \jacobian(0) \right\|_2^2 \leq \frac{96^{1/3} n |\target - \pred(0) \|_2^{2/3}}{\vars^{2/3}\lambda_0^{1/3}\delta^{2/3}m^{1/3}}$.
Therefore, we know that Condition~\ref{con:stable-jaco} holds if $m = \Omega\left( \frac{n^3 \|\target - \pred(0) \|_2^2}{\vars^2 \lambda_0^4 \delta^2} \right)$. Furthermore, by Assumption~\ref{ass:data} we have
\begin{equation}
    \expect \left[\left\|\target - \pred(0) \right\|_2^2 \right] = \|\target \|_2^2 + 2 \target^\top \expect \left[\pred(0) \right] + \expect\left[\|\pred(0) \|_2^2 \right] = \bigo \left(n \right) .
\end{equation}

Thus by Markov's inequality, we have probability at least $1 - \delta$, $\|\target - \pred(0) \|_2^2 = \bigo\left(\frac{n}{\delta} \right)$. 
Thus the condition on $m$ can be written as $m = \Omega \left(\frac{n^4}{\vars^2 \lambda_0^4 \delta^3} \right)$. We note that the larger $m$ is, the smaller the constant $C$ in Condition~\ref{con:stable-jaco} is ($C \sim \bigo\left(m^{-1/6} \right)$). We now finish the proof.

\begin{proof}[Proof of Lemma~\ref{lem:full-rank}]
Notice that $\gram(0)$ can be written as the sum of random symmetric matrices:
\begin{equation}
    \gram(0) = \sum_{r=1}^{m} \gram(\weight_r)
\end{equation}
where $\gram_{ij}(\weight_r) = \frac{1}{m}  \inputs_i^\top \inputs_j \indicator \left\{ \weight_r^\top\inputs_i\geq 0, \weight_r^\top \inputs_j \geq 0 \right\}$. Furthermore, $\gram(\weight_r)$ are positive semi-definite and $\|\gram(\weight_r)\|_2 \leq \trace\left(\gram(\weight_r)\right) = \frac{n}{m}$. Therefore, we can apply the matrix Chernoff bound (e.g., \citet{tropp2015introduction}), giving us the following bound:
\begin{equation}\label{eq:chernoff}
    \prob \left[\lambda_{\mathrm{min}}(\gram(0)) \leq (1 - \frac{1}{4})\lambda_0 \right] \leq n \exp\left(- \frac{1}{4^2} \frac{\lambda_0}{n/m}\right)
\end{equation}
Letting the RHS of eqn.~\eqref{eq:chernoff} be $\delta$, we have $m = \bigo \left(\frac{n}{\lambda_0} \log \frac{n}{\delta} \right)$.
\end{proof}

\begin{proof}[Proof of Lemma~\ref{lem:jaco-pert}]
Let $[\mathbf{v}_r]_{k-}$ denote the $k$-th smallest entry of $\{\mathbf{v}_1, ..., \mathbf{v}_m \}$ after sorting its entries in terms of absolute value. We first state a intermediate lemma we prove later.
\begin{lem}\label{lem:jaco-pert-2}
Given an integer $k$, suppose 
\begin{equation}
    \left\|\weight - \weight(0)\right\|_2 \leq \sqrt{k} \left[\weight_r(0)^\top \inputs_i \right]_{k-}
\end{equation}
holds for $i = 1, 2, ..., n$. Then, we have
\begin{equation}
    \left\|\jacobian - \jacobian(0) \right\|_2^2 \leq \frac{2nk}{m}
\end{equation}
\end{lem}
By taking $k = \frac{R^{2/3}m^{2/3}}{\vars^{2/3}\delta^{2/3}}$, we have
$ \left\|\jacobian - \jacobian(0) \right\|_2^2 \leq \frac{2nR^{2/3}}{\vars^{2/3}\delta^{2/3}m^{1/3}}$.
To complete the proof, all that remains is to prove $R \leq \sqrt{k} \left[\weight_r(0)^\top \inputs_i \right]_{k-}$. 

Observe that $\weight_r(0)^\top \inputs_i$ for $r \in [m]$ and $i \in [n]$ all distribute as $\normal(0, \vars^2)$, though they depend on each other. We begin by proving that with probability at least $1 - \delta$, at most $k$ of the hidden units, we have $|\weight_r(0)^\top \inputs_i| \leq \frac{k\vars\delta}{m}$. 
To this aim, we define $\gamma_\alpha$ to be the number for which $\prob \left[|g| \leq \gamma_\alpha \right] = \alpha$ where $g \sim \normal(0, \vars^2)$. By anti-concentration of Gaussian, $\gamma_\alpha$ trivially obeys $\gamma_\alpha \geq \sqrt{\pi/2} \alpha \vars$. Now let $\alpha = \frac{k\delta}{m}$. We note that
\begin{equation}
    \expect \left[\sum_{r=1}^m \indicator\left\{|\weight_r(0)^\top \inputs_i| \leq \gamma_\alpha \right\} \right] = \sum_{r=1}^m \prob \left[|\weight_r(0)^\top \inputs_i| \leq \gamma_\alpha \right] = k\delta
\end{equation}
By applying Markov's inequality, we obtain 
\begin{equation}
    \prob\left[\sum_{r=1}^m \indicator\left\{|\weight_r(0)^\top \inputs_i| \leq \gamma_\alpha \right\} \geq k\right] \leq \delta
\end{equation}
Therefore, we have $\sqrt{k} \left[\weight_r(0)^\top \inputs_i \right]_{k-} \geq \frac{k^{3/2}\vars\delta}{m} = R$.
\end{proof}

\begin{proof}[Proof of Lemma~\ref{lem:jaco-pert-2}]
We prove the lemma by contradiction. First, we define the event
\begin{equation}
    A_{ir} = \left\{\indicator\{\inputs_i^\top \weight_r \geq 0\} \neq \indicator\{\inputs_i^\top \weight_r(0) \geq 0\} \right\}
\end{equation}
We can then bound the Jacobian perturbation.
    \begin{equation}\label{eq:hypo}
    \begin{aligned}
        \left\|\jacobian - \jacobian(0) \right\|_2^2 \leq \left\| \jacobian - \jacobian(0) \right\|_\fisher^2 \leq \frac{1}{m} \sum_{i=1}^n\sum_{r=1}^m \indicator\left\{A_{ir} \right\} = \frac{2nk}{m}
    \end{aligned}
    \end{equation}
If eqn.~\eqref{eq:hypo} does not hold, then we know there must exist an $\inputs_i$ with $i \in [n]$ such that $\indicator\left\{A_{ir} \right\}$ with $r\in[m]$ has at least $2k$ non-zero entries. Let $\{(a_r, b_r) \}_{r=1}^{2k}$ be entries of $\weight_r^\top \inputs_i$ and $\weight_r(0)^\top \inputs_i$ at these non-zero locations respectively, then we have (by definition, $|b_r| \geq \left\{\weight_r(0)^\top \inputs_i \right\}_{k-}$)
\begin{equation}
\begin{aligned}
    \left\|\weight - \weight(0) \right\|_2^2 &\geq \sum_{r=1}^m |\weight_r^\top \inputs_i - \weight_r(0)^\top \inputs_i|^2 \\
    & \geq \sum_{r=1}^{2k} |a_r - b_r|^2 \geq \sum_{r=1}^{2k} |b_r|^2 \\
    & \geq k \left\{\weight_r(0)^\top \inputs_i \right\}_{k-}^2
\end{aligned}
\end{equation}
The second inequality, we used the fact that $\mathrm{sign}(a_r) \neq \mathrm{sign}(b_r)$.
\end{proof}

\subsection{Proof of Theorem~\ref{thm:kfac}}
Based on the update in the weight space, we can get the update in output space accordingly.
\begin{equation}\label{eq:kfac-update}
    \begin{aligned}
        \pred(k+1) - \pred(k) &= \pred\left(\weight(k) - \lr \fisher_\mathrm{K-FAC}^{-1}(k)\jacobian(k)^\top (\pred(k) - \target)\right) - \pred(\weight(k)) \\
        &= -\int_{s=0}^1 \left\langle \frac{\partial \pred\left(\weight(s) \right)}{\partial \weight^\top}, \: \lr\fisher_\mathrm{K-FAC}^{-1}(k) \jacobian(k)^\top(\pred(k) - \target) \right\rangle ds \\
        &= \underbrace{-\int_{s=0}^1 \left\langle \frac{\partial \pred\left(\weight(k)\right)}{\partial \weight^\top}, \: \lr\fisher_\mathrm{K-FAC}^{-1}(k)\jacobian(k)^\top(\pred(k) - \target) \right\rangle ds}_{\circled{1}} \\
        & \quad + \underbrace{\int_{s=0}^1 \left\langle \frac{\partial \pred\left(\weight(k)\right)}{\partial \weight^\top} - \frac{\partial \pred\left(\weight(s)\right)}{\partial \weight^\top}, \: \lr\fisher_\mathrm{K-FAC}^{-1}(k)\jacobian(k)^\top(\pred(k) - \target) \right\rangle ds}_{\circled{2}} \\
    \end{aligned}
\end{equation}
We first analyze the first term $\circled{1}$ by expanding $\fisher_\mathrm{K-FAC}(k)$:
\begin{equation}
\begin{aligned}
    \circled{1} &= \lr \left(\bX \ast \bS\right) \left( (\bX^\top\bX)^{-1} \otimes (\bS^\top \bS)^{-1} \right) \left(\bX^\top \star \bS^\top \right)   (\target - \pred(k)) \\
    & = \lr \left(\bX(\bX^\top\bX)^{-1}\bX^\top \odot \bS(\bS^\top\bS)^{-1}\bS^\top \right) (\target - \pred(k)) \\
    & = \lr \left(\bX(\bX^\top\bX)^{-1}\bX^\top \odot \iden \right) (\target - \pred(k))
\end{aligned}
\end{equation}
The first equality we used properties of Khatri-Rao, Hadamard and Kronecker products while the second equality used the generalized inverse.

Next, we need to show $\circled{2}$ is small (negligible) compared to $\circled{1}$. Similar to exact natural gradient, we first assume two conditions hold. However, the first condition is slightly different in the sense that we assume the positive definiteness of $\bS\bS^\top$ (instead of $\gram$). We also need a stronger version of Condition~\ref{con:stable-jaco} in the sense that the Jacobian matrix is stable for two consecutive steps $ \left\|\int_{s=0}^1 \jacobian(\weight(s)) - \jacobian(\weight(k)) ds \right\|_2 \leq \sqrt{\lambda_\mathrm{min}(\bS\bS^\top)} \frac{\lambda_\mathrm{min}(\bX\bX^\top)}{\lambda_\mathrm{max}(\bX\bX^\top)}$. With these two conditions, we are ready to bound the term $\circled{2}$.
\begin{equation}
\begin{aligned}
    \circled{2} &\leq \lr\sqrt{\lambda_\mathrm{min}(\bS\bS^\top)} \frac{\lambda_\mathrm{min}(\bX\bX^\top)}{\lambda_\mathrm{max}(\bX\bX^\top)} \left((\bX^\top \bX^\top)^{-1} \otimes (\bS^\top\bS)^{-1} \right) \left(\bX^\top \star \bS^\top \right) (\target - \pred(k)) \\
    & = \lr\sqrt{\lambda_\mathrm{min}(\bS\bS^\top)} \frac{\lambda_\mathrm{min}(\bX\bX^\top)}{\lambda_\mathrm{max}(\bX\bX^\top)} \left( (\bX^\top \bX^\top)^{-1}\bX^\top \star (\bS^\top\bS)^{-1}\bS^\top \right)(\target - \pred(k)) \\
    & = \lr\sqrt{\lambda_\mathrm{min}(\bS\bS^\top)} \frac{\lambda_\mathrm{min}(\bX\bX^\top)}{\lambda_\mathrm{max}(\bX\bX^\top)}\left( (\bX^\top \bX^\top)^{-1}\bX^\top \star \bS^\top(\bS\bS^\top)^{-1} \right)(\target - \pred(k)) \\
\end{aligned}
\end{equation}
The key of bounding $\circled{2}$ is to analyze $(\bX^\top \bX^\top)^{-1}\bX^\top \star \bS^\top(\bS\bS^\top)^{-1}$. For convenience, we denote this term as $\circled{3}$. By the identity of $(\mathbf{A} \ast \mathbf{B}) (\mathbf{A}^\top \star \mathbf{B}^\top) = \mathbf{A}\mathbf{A}^\top \odot \mathbf{B}\mathbf{B}^\top$, we have
\begin{equation}
\begin{aligned}
    \sigma_\mathrm{max}(\circled{3}) &= \sqrt{\lambda_\mathrm{max}(\circled{3}^\top \circled{3})} \\
    & = \sqrt{\lambda_\mathrm{max}\left(\bX (\bX^\top \bX)^{-1}(\bX^\top \bX)^{-1} \bX^\top \odot (\bS\bS^\top)^{-1} \right)}
\end{aligned}
\end{equation}
According to Lemma~\ref{lem:hadamard}, we have
\begin{equation}
    \sigma_\mathrm{max}\left(\circled{3} \right) \leq \sqrt{\frac{1}{\lambda_\mathrm{min}(\bX^\top\bX)^2} \frac{1}{\lambda_\mathrm{min}(\bS\bS^\top)}}
\end{equation}
Also, we have 
\begin{equation}
    \lambda_\mathrm{min}\left(\bX (\bX^\top\bX)^{-1}\bX^\top \odot \iden \right) \geq \frac{1}{\lambda_\mathrm{max}\left(\bX^\top \bX \right)}
\end{equation}
By choosing a slightly larger $m$, we can show that $\circled{1} \gg \circled{2}$. Therefore, we can safely ignore $\circled{2}$ in eqn.~\eqref{eq:kfac-update} and get
\begin{equation}
\begin{aligned}
    \left\|\target - \pred(k+1) \right\|_2^2 &= \|\target - \pred(k) - (\pred(k+1) - \pred(k))\|_2^2 \\
    &= \|\target - \pred(k)\|_2^2 - 2 (\target - \pred(k))^\top(\pred(k+1) - \pred(k)) + \|\pred(k+1) - \pred(k)\|_2^2 \\
    & \approx \|\target - \pred(k)\|_2^2 - 2\lr \left(\target - \pred(k)\right)^\top\left(\bX(\bX^\top\bX)^{-1}\bX^\top \odot \iden \right) \left(\target - \pred(k)\right) \\
    &\quad + \lr^2 \left(\target - \pred(k)\right)^\top\left(\bX(\bX^\top\bX)^{-1}\bX^\top \odot \iden \right)^2 \left(\target - \pred(k)\right) \\
    & \leq \|\target - \pred(k)\|_2^2 - \lr \left(\target - \pred(k)\right)^\top\left(\bX(\bX^\top\bX)^{-1}\bX^\top \odot \iden \right) \left(\target - \pred(k)\right) \\
    & \leq \left(1 - \frac{\lr }{\lambda_\mathrm{max}(\bX^\top\bX)}\right) \|\target - \pred(k)\|_2^2 .
\end{aligned}
\end{equation}
In the last second inequality we used the fact that $\lambda_\mathrm{max}\left(\bX(\bX^\top\bX)^{-1}\bX^\top \odot \iden \right) \leq \frac{1}{\lambda_\mathrm{min}\left(\bX^\top\bX\right)}$ and the step size $\lr = \bigo\left(\lambda_\mathrm{min}\left(\bX^\top\bX\right) \right)$.

Next, we move on to show the weights of the network remain close to the initialization point.
\begin{equation}
\begin{aligned}
    \left\|\weight(k+1) - \weight(0) \right\|_2 &\leq \lr \sum_{s=0}^k \left\|\left( (\bX^\top \bX^\top)^{-1}\bX^\top \star \bS(s)^\top(\bS(s)\bS(s)^\top)^{-1} \right) (\target - \pred(k)) \right\|_2 \\
    & \leq \lr \sum_{s=0}^k \frac{\left\|\target - \pred(k) \right\|_2}{\sqrt{\lambda_{\bS}/2}} \frac{1}{\lambda_\mathrm{min}\left(\bX^\top\bX \right)} \\
    & \leq \lr \sum_{s=0}^k (1 - \frac{\lr}{\lambda_\mathrm{max}\left(\bX^\top\bX \right)})^{s/2}\frac{\left\|\target - \pred(0) \right\|_2}{\sqrt{\lambda_{\bS}/2}} \frac{1}{\lambda_\mathrm{min}\left(\bX^\top\bX \right)} \\
    &\leq \frac{2\left\|\target - \pred(0)\right\|_2}{\sqrt{\lambda_{\bS}/2}} \frac{\lambda_\mathrm{max}\left(\bX^\top\bX \right)}{\lambda_\mathrm{min}\left(\bX^\top\bX \right)}
\end{aligned}
\end{equation}
Based on matrix pertubation analysis (similar to Lemma~\ref{lem:jaco-pert}), it is easy to show that if $m = \bigo\left(\frac{n^4}{\vars^2\lambda_{\bS}^4 \kappa_{\bX^\top\bX}^4 \delta^3} \right)$ and step size $\lr = \bigo \left(\lambda_\mathrm{min}\left(\bX^\top \bX\right) \right)$, we have
\begin{equation}
    \lambda_\mathrm{min}\left(\bS\bS^\top \right) \geq \frac{\lambda_{\bS}}{2} \; \mathrm{and} \;
    \left\|\int_{s=0}^1 \jacobian(\weight(s)) - \jacobian(\weight(k)) ds \right\|_2 \leq \sqrt{\lambda_\mathrm{min}(\bS\bS^\top)} \frac{\lambda_\mathrm{min}(\bX\bX^\top)}{\lambda_\mathrm{max}(\bX\bX^\top)}
\end{equation}

\subsection{Proof of Positive Definiteness of $\bS^\infty{\bS^\infty}^\top$}
\label{app:pd-pre-act}
Following~\citet{du2018gradient}'s proof that $\gram^\infty$ is strictly positive definite, we apply the same argument to $\bS^\infty{\bS^\infty}^\top$. Recall that $\bS = [\phi^\prime(\bX \weight_1), ..., \phi^\prime(\bX \weight_m)]$ which is the pre-activation derivative. In our setting, $\phi^\prime(\weight^\top\inputs) = \indicator\{\weight^\top \inputs \geq 0\}$.
For each $\inputs_i \in \real^d$, it induces an infinite-dimensional feature map $\phi^\prime(\inputs_i) \in \hilbert$, where $\hilbert$ is the Hilbert space of integrable $d$-dimensional vector fields on $\real^d$.

Now to prove that $\bS^\infty{\bS^\infty}^\top$ is strictly positive definite, it is equivalent to show $\phi^\prime(\inputs_1), ..., \phi^\prime(\inputs_n) \in \hilbert$ are linearly independent. Suppose that there are $\alpha_1, ..., \alpha_n \in \real$ such that
\begin{equation}
	\alpha_1 \phi^\prime(\inputs_1) + ... + \alpha_n \phi^\prime(\inputs_n) = 0
\end{equation} 
We now prove that $\alpha_i = 0$ for all $i$.

We define $D_i = \left\{\weight \in \real^d: \weight^\top \inputs_i = 0 \right\}$. As shown by~\citet{du2018gradient}, $D_i \not\subset \cup_{j\neq i} D_j$. For a fixed $i \in [n]$, we can choose $\bz \in D_i \setminus \cup_{j \neq i} D_j$. We can pick a small enough radius $r_0 > 0$ such that $B(\bz, r) \cap D_j = \emptyset, \forall j \neq i, r \leq r_0$. Let $B(\bz, r) = B_r^+ \cup B_r^-$, where $B_r^+ = B(\bz, r) \cap D_i$.

For $j \neq i$, $\phi^\prime(\inputs_j)$ is continuous in the neighborhood of $\bz$, therefore for any $\epsilon > 0$ there is a small enough $r$ such that 
\begin{equation}
	\forall \weight \in B(\bz, r), \phi^\prime(\weight^\top \inputs_j) = \phi^\prime(\bz^\top \inputs_j)
\end{equation}
Let $\mu$ be Lebesgue measure on $\real^d$, we have
\begin{equation}
	\lim_{r\rightarrow 0^+} \frac{1}{\mu(B_r^+)} \int_{B_r^+ } \phi^\prime(\weight^\top\inputs_j) d\weight = \lim_{r\rightarrow 0^+} \frac{1}{\mu(B_r^-)} \int_{B_r^-} \phi^\prime(\weight^\top\inputs_j) d\weight = \phi^\prime(\bz^\top \inputs_j)
\end{equation}
Now recall that $\sum_i \alpha_i \phi^\prime (x_i) \equiv 0$, we have
\begin{equation}
\begin{aligned}
	0 &= \lim_{r\rightarrow 0^+} \frac{1}{\mu(B_r^+)} \int_{B_r^+ } \sum_j \alpha_j \phi^\prime(\weight^\top\inputs_j) d\weight - \lim_{r\rightarrow 0^+} \frac{1}{\mu(B_r^-)} \int_{B_r^-} \sum_j \alpha_j \phi^\prime(\weight^\top\inputs_j) d\weight \\
	& = \sum_j \alpha_j \left(\lim_{r\rightarrow 0^+} \frac{1}{\mu(B_r^+)} \int_{B_r^+ } \phi^\prime(\weight^\top\inputs_j) d\weight - \lim_{r\rightarrow 0^+} \frac{1}{\mu(B_r^-)} \int_{B_r^-}  \phi^\prime(\weight^\top\inputs_j) d\weight \right) \\
	& = \sum_j \alpha_j \delta_{ij} = \alpha_i
\end{aligned}
\end{equation}
where $\delta_{ij}$ is the Kronecker delta. We complete the proof.

\subsection{Proof of Theorem~\ref{thm:bound-lambda}}
It has been shown that the Gram matrix for infinite width networks has the following form~\citep{xie2016diverse, arora2019fine}:
\begin{equation}
    \gram_{ij}^\infty = \left(\frac{\pi - \arccos(\inputs_i^\top\inputs_j)}{2\pi} \right) \inputs_i^\top \inputs_j \triangleq h(\inputs_i, \inputs_j)
\end{equation}
We note that any function defined on the unit sphere has a spherical harmonic decomposition:
\begin{equation}\label{eq:harmonic}
    h(\inputs_i, \inputs_j) = \sum_{u=1}^\infty \gamma_u \phi(\inputs_i) \phi(\inputs_j)
\end{equation}
where $\phi_u(\inputs): \sphere^{d-1} \rightarrow \real$ is a spherical harmonic basis. In eqn.~\eqref{eq:harmonic}, we sort the spectrum $\gamma_u$ by magnitude. As shown by~\citet{xie2016diverse}, $\gamma_n = \Omega\left( n^{\beta - 1} \right)$, where $\beta \in (0, 0.5)$. Then we can also define the truncated function:
\begin{equation}
   \left[\gram_{ij}^\infty\right]^n = h^{n}(\inputs_i, \inputs_j) = \sum_{u=1}^n \gamma_u \phi(\inputs_i) \phi(\inputs_j)
\end{equation}
Due to the fact that $\gram^\infty - \left[\gram^\infty\right]^n$ is PSD, we can then bound the smallest eigenvalue of $\left[\gram^\infty\right]^n$. Define a matrix $\mathbf{K} \in \real^{n \times n}$ whose rows are
\begin{equation}
    \mathbf{K}^i = \left[ \sqrt{\gamma_1}\phi_1(\inputs_i), ...,  \sqrt{\gamma_n}\phi_n(\inputs_i)\right]
\end{equation}
for $1 \leq i \leq n$. It is easy to show $
    \left[\gram^\infty\right]^n = \mathbf{K} \mathbf{K}^\top$
and $\lambda_\mathrm{min} (\mathbf{K} \mathbf{K}^\top) = \lambda_\mathrm{min} (\mathbf{K}^\top \mathbf{K})$. So we only need to bound $\lambda_\mathrm{min} (\mathbf{K}^\top \mathbf{K})$. Observe that $\mathbf{K}^\top \mathbf{K}$ is the sum of $n$ random matrices ${\mathbf{K}^i}^\top \mathbf{K}^i$. Based on the properties of spherical harmonic basis $\expect \left[\phi_i(\inputs)\phi_j(\inputs) \right] = \delta_{ij}$, we have $\lambda_\mathrm{min}\left(\expect\left[{\mathbf{K}^i}^\top \mathbf{K}^i \right]\right) = n\gamma_n$. Observe that all random matrices $\mathbf{K}^\top \mathbf{K}$ are PSD and upper-bounded in the sense
\begin{equation}
    \|{\mathbf{K}^i}^\top \mathbf{K}^i \|_2 \leq \trace({\mathbf{K}^i}^\top \mathbf{K}^i ) = \sum_{u=1}^n \gamma_u \phi_u(\inputs_i)^2 \leq h(\inputs_i, \inputs_i) = \frac{1}{2}
\end{equation}
Therefore, the matrix Chernoff bound gives
\begin{equation}
    \prob \left[ \lambda_\mathrm{min}(\left[\gram^\infty\right]^n) \leq (1 - \frac{1}{2}) \lambda_\mathrm{min}\left(\expect\left[\mathbf{K}^\top \mathbf{K} \right]\right) \right] \leq n \exp \left(- \frac{1}{2^2}\lambda_\mathrm{min}\left(\expect\left[\mathbf{K}^\top \mathbf{K} \right]\right) \right)
\end{equation}
According to~\citet{xie2016diverse}, $\gamma_n$ decays slower than $\bigo \left(1/n \right)$. This completes the proof.

\section{Proofs for Section~\ref{sec:sec-5}}
\subsection{Proof of Theorem~\ref{thm:generalization}}\label{app:thm-generalization}
Similar to~\citet{neyshabur2018the, arora2019fine}, we analyze the generalization error based on Rademacher Complexity~\citep{bartlett2002rademacher}. Based on Rademacher complexity, we have the following generalization bound.

\begin{lem}[Generalization bound with Rademacher Complexity~\citep{mohri2018foundations}]\label{lem:rade-bound}
Suppose the loss function $\ell(\cdot, \cdot)$ is bounded in $[0, c]$ and is $\rho$-Lipschitz in the first argument. Then with probability at least $1 - \delta$ over sample $\tdata$ of size $n$:
\begin{equation}
    \sup_{f\in\mathcal{F}} \left\{ \loss_\data(f) - \loss_\tdata(f) \right\} \leq 2\rho \rade_\tdata(\mathcal{F}) + 3c \sqrt{\frac{\log(2/\delta)}{2n}}
\end{equation}
\end{lem}
Therefore, to get the generalization bound, we only need to calculate the Rademacher complexity of a certain function class. Lemma~\ref{lem:restricted-class} suggests that the learned function $f(\weight, \ba)$ from NGD is in a restricted class of neural networks whose weights are close to the initialization $\weight(0)$. The following lemma bounds the Rademacher complexity of this function class.
\begin{lem}[Rademacher Complexity for a Restricted Class Neural Nets~\citep{arora2019fine}]\label{lem:rade-comp}
For given $A, B >0$, with probability at least $1 - \delta$ over the random initialization, the following function class
\begin{equation}\label{eq:fun-class}
    \mathcal{F}_{A,B} = \left\{f(\weight, \ba): \forall r\in[m], \|\weight_r - \weight_r(0) \|_2 \leq A, \|\weight - \weight(0)\|_2 \leq B \right\}
\end{equation}
has empirical Rademacher complexity bounded as
\begin{equation}
    \rade_\tdata(\mathcal{F}_{A,B}) \leq \frac{B}{\sqrt{2n}} \left(1+\left(\frac{2\log\frac{2}{\delta}}{m} \right)^{1/4} \right) + 2A^2\sqrt{m} + A\sqrt{2\log\frac{2}{\delta}}
\end{equation}
\end{lem}
With Lemma~\ref{lem:rade-bound} and~\ref{lem:rade-comp} at hand, we are only left to bound the distance of the weights to their initialization. Recall the update rule for $\weight$:
\begin{equation}
    \weight(k+1) = \weight(k) - \lr \jacobian(k)^\top \gram(k)^{-1} (\pred(k) - \target)
\end{equation}
The following lemma upper bounds the distance for each hidden unit.
\begin{lem}\label{lem:single-weight}
If two conditions hold for $s = 0, ..., k$, then we have
\begin{equation}
    \|\weight_r(k+1) - \weight_r(0)\|_2 \leq \frac{4\sqrt{n}\|\target - \pred(0)\|_2}{\sqrt{m}\lambda_0} = \bigo\left( \frac{n}{\lambda_0 m^{1/2}\delta^{1/2}} \right)
\end{equation}
\end{lem}
To analyze the whole weight vector, we start with ideal case -- infinite width network. In that case, both $\jacobian$ and $\gram$ are constant matrix throughout the training, and the function error decay exponentially. It is easy to show that the distance is given by
\begin{equation}\label{eq:ideal-dist}
    \|\weight(\infty) - \weight(0)\|_2 = \sqrt{(\target - \pred(0))^\top \left(\gram^\infty\right)^{-1} (\target - \pred(0))}
\end{equation}
In the case of finite wide networks, $\jacobian$ and $\gram$ would change along the weights, but we can bound the changes and show the norms are small if the network are wide enough. Therefore, the distance is dominated by eqn.~\eqref{eq:ideal-dist}. From Lemma~\ref{lem:restricted-class}, we know $\|\weight(k+1) - \weight(0)\|_2 = \bigo\left(\sqrt{\frac{n}{\lambda_0 \delta }}\right)$. According to Lemma~\ref{lem:jaco-pert}, it is easy to show that $\|\jacobian(k) - \jacobian(0) \|_2 = \bigo\left(\frac{n^{2/3}}{\vars^{1/3}\lambda_0^{1/6} m^{1/6} \delta^{1/2}} \right)$ and $\|\gram(k) - \gram(0) \|_2 = \bigo \left( \frac{n^{4/3}}{\vars^{2/3}\lambda_0^{1/3} m^{1/3} \delta } \right)$. 
With these bounds at hand, we are ready to bound $\|\weight(\infty) - \weight(0) \|_2$ for finite wide networks.

\begin{lem}\label{lem:gen-dist}
    Under the same setting as Theorem~\ref{thm:d-convergence}, with probability at least $1 - \delta$ over the random initialization, we have
    \begin{equation}
    \|\weight - \weight(0)\|_2 \leq \sqrt{\target^\top (\gram^\infty)^{-1} \target} + \bigo \left( \sqrt{\frac{n}{\lambda_0 \delta}} \vars \right) + \frac{\mathrm{poly}(n, \frac{1}{\lambda_0}, \frac{1}{\delta})}{m^{1/4}}
    \end{equation}
\end{lem}

Finally, we know that for any sample $\tdata$ drawn from data distribution $\data$, with probability at least $1 - \delta / 3$ over the random initialization, the following hold simutaneously:
\begin{enumerate}
    \item Optimization succeeds (Theorem~\ref{thm:d-convergence}):
        \begin{equation}
            \left\| \pred(k) - \target \right\|_2^2 \leq (1 - \lr)^k \cdot \bigo\left( \frac{n}{\delta} \right) \leq \frac{n\epsilon^2}{4}
        \end{equation}
        This implies an upper bound on the training error:
        \begin{equation}
        \begin{aligned}
            \loss_\tdata(f) &= \frac{1}{n} \sum_{i=1}^n \left[ \ell(u_i(k), y_i) \right] \leq \frac{1}{n} \sum_{i=1}^n | u_i(k) - y_i | \\
            &\leq \frac{1}{\sqrt{n}} \left\| \pred(k) - \target \right\|_2 = \frac{\epsilon}{2}
        \end{aligned}
        \end{equation}
    \item The learned function $f(\weight, \ba)$ belongs to the restricted function class~\eqref{eq:fun-class}.
    \item The function class $\mathcal{F}_{A,B}$ has Rademacher complexity bounded as
        \begin{equation}
        \begin{aligned}
            \rade_\tdata\left( \mathcal{F}_{A,B} \right) &\leq \sqrt{\frac{\target^\top (\gram^\infty)^{-1}\target}{2n}} + \bigo \left(\frac{\vars}{\sqrt{\lambda_0 \delta}} \right) + \frac{\mathrm{poly}(n, \frac{1}{\lambda_0}, \frac{1}{\delta})}{\vars^{1/2}m^{1/4}} \\
            &= \sqrt{\frac{\target^\top (\gram^\infty)^{-1}\target}{2n}} + \frac{\epsilon}{4}
        \end{aligned}
        \end{equation}
\end{enumerate}

Also, with the probability at least $1 - \delta/3$ over the sample $\tdata$, we have
\begin{equation}
    \sup_{f\in\mathcal{F}} \left\{ \loss_\data(f) - \loss_\tdata(f) \right\} \leq 2 \rade_\tdata(\mathcal{F}) + 3 \sqrt{\frac{\log(6/\delta)}{2n}}
\end{equation}
Taking a union bound, we have know that with probability at least $1 - \frac{2}{3}\delta$ over the sample $\tdata$ and the random initialization, we have
\begin{equation}
    \sup_{f\in\mathcal{F}} \left\{ \loss_\data(f) - \loss_\tdata(f) \right\} \leq \sqrt{\frac{2\target^\top (\gram^\infty)^{-1}\target}{n}} + 3\sqrt{\frac{\log(6/\delta)}{2n}} + \frac{\epsilon}{2}
\end{equation}
which implies
\begin{equation}
   \loss_\data(f) \leq \sqrt{\frac{2\target^\top (\gram^\infty)^{-1}\target}{n}} + 3\sqrt{\frac{\log(6/\delta)}{2n}} + \epsilon
\end{equation}

\subsection{Technical Proofs for Generalization Analysis}
\begin{proof}[Proof of Lemma~\ref{lem:single-weight}]
    From Theorem~\ref{thm:d-convergence}, we have 
    \begin{equation}
        \left\|\target - \pred(k) \right\|_2 \leq \sqrt{(1 - \lr)^k} \left\|\target - \pred(0) \right\|_2 \leq (1 - \frac{\lr}{2})^k \left\|\target - \pred(0) \right\|_2
    \end{equation}
    Recall the natural gradient update rule:
    \begin{equation}
        \weight(k+1) = \weight(k) - \lr \jacobian(k)^\top \gram(k)^{-1} (\pred(k) - \target)
    \end{equation}
    By $\lambda_{\mathrm{min}}(\gram(k)) \geq \frac{\lambda_0}{2}$, we have $\|\gram(k)^{-1} (\pred(k) - \target)\|_2 \leq \frac{2}{\lambda_0}\|\pred(k) - \target\|_2$ and then
    \begin{equation}
        \left\|\weight_r(k+1) - \weight_r(k)\right\|_2 \leq \frac{2\lr \sqrt{n}}{\sqrt{m }\lambda_0} \left\|\pred(k) - \target \right\|_2
    \end{equation}
    Therefore, we have
    \begin{equation}
    \begin{aligned}
        \left\|\weight_r(k) - \weight_r(0)\right\|_2 &\leq \sum_{s=0}^{k-1}\left\|\weight_r(s+1) - \weight_r(s)\right\|_2 \leq \sum_{s=0}^{k-1} \frac{2\lr \sqrt{n}}{\sqrt{m }\lambda_0} \left\|\pred(s) - \target \right\|_2 \\
        &\leq \frac{2\lr \sqrt{n}}{\sqrt{m }\lambda_0} \sum_{s=0}^{k-1} (1- \frac{\lr}{2})^s \left\|\pred(0) - \target \right\|_2 \\
        & \leq \frac{2\lr \sqrt{n}}{\sqrt{m }\lambda_0} \frac{2}{\lr} \left\|\pred(0) - \target \right\|_2 = \bigo \left(\frac{n}{\lambda_0 m^{1/2}\delta^{1/2}} \right)
    \end{aligned}
    \end{equation}
\end{proof}

\begin{proof}[Proof of Lemma~\ref{lem:gen-dist}]
To bound the norm of distance, we first decompose the total distance into the sum of each weight update.
\begin{equation}\label{eq:decomp}
\begin{aligned}
    \weight(k) - \weight(0)
    &= \sum_{s=0}^{k-1} \left(\weight(s+1) - \weight(s) \right) \\
    &= - \sum_{s=0}^{k-1}\lr \jacobian(s)^\top \gram(s)^{-1} (\pred(s) - \target)
\end{aligned}
\end{equation}
We then analyze the term $\pred(k) - \target$, which evolves as follows.
\begin{lem}\label{lem:norm-pred-error}
Under the same setting as Theorem~\ref{thm:d-convergence}, with probability at least $1 - \delta$ over the random initialization, we have
\begin{equation}\label{eq:output-dyn}
    \pred(k) - \target = -(1 - \lr)^k \target + \mathbf{\zeta}(k)
\end{equation}
where $\| \mathbf{\zeta}(k) \|_2 = \bigo \left( (1-\lr)^k \sqrt{\frac{n}{\delta}}\vars + k \left(1- \frac{\lr}{2}\right)^{k-1}\lr\frac{n^{7/6}}{\vars^{1/3}\lambda_0^{2/3}m^{1/6}\delta} \right)$
\end{lem}

Plugging eqn.~\eqref{eq:output-dyn} into eqn.~\eqref{eq:decomp}, we have
\begin{equation}
    \weight(k) - \weight(0) = \sum_{s=0}^{k-1} \lr \jacobian(s)^\top \gram(s)^{-1}(1-\lr)^s \target - \lr\jacobian(s)^\top \gram(s)^{-1}\mathbf{\zeta}(s)
\end{equation}

The RHS term in above equation is considered perturbation and we can upper bound their norm easily. By Lemma~\ref{lem:norm-pred-error}, we have 
\begin{equation}\label{eq:error-1}
    \|\jacobian(s)^\top \gram(s)^{-1}\zeta(s) \|_2 = \bigo \left( (1-\lr)^s \sqrt{\frac{n}{\lambda_0 \delta}}\vars + s \left(1- \frac{\lr}{2}\right)^{s-1} \lr \frac{n^{7/6}}{\vars^{1/3}\lambda_0^{7/6}m^{1/6}\delta} \right)
\end{equation}
Plugging eqn.~\eqref{eq:error-1} into eqn.~\eqref{eq:decomp}, we have
\begin{equation}\label{eq:error-decomp}
\begin{aligned}
    \weight(k) - \weight(0) &= \sum_{s=0}^{k-1} \lr \jacobian(s)^\top  \gram(s)^{-1}(1-\lr)^s\target + \error_1 \\
    & = \sum_{s=0}^{k-1} \lr \jacobian(s)^\top (\gram^\infty)^{-1}(1-\lr)^s\target + \error_1 + \error_2 \\
    &= \sum_{s=0}^{k-1} \lr \jacobian(0)^\top  (\gram^\infty)^{-1}(1-\lr)^s\target + \error_1 + \error_2 + \error_3 \\
\end{aligned}
\end{equation}
For $\error_1$, we have $\|\error_1\|_2 = \bigo\left( \sqrt{\frac{n}{\lambda_0 \delta}}\vars + \frac{n^{7/6}}{\vars^{1/3}\lambda_0^{7/6}m^{1/6}\delta}\right)$ by using the following inequality:
\begin{equation}
    \sum_{s=0}^{k-1} s \left(1 - \frac{\lr}{2} \right)^{s-1} \leq \sum_{s=0}^{\infty} s \left(1 - \frac{\lr}{2} \right)^{s-1} = \frac{4}{\lr^2}
\end{equation}
For $\error_2$, we have
\begin{equation}\label{eq:error-2}
\begin{aligned}
    \|\error_2 \|_2 &\leq \sum_{s=0}^{k-1} \lr (1-\lr)^s\|\jacobian(s)\|_2 \| \gram(s)^{-1} - (\gram^\infty)^{-1} \|_2 \|\target \|_2 
\end{aligned}
\end{equation}
In eqn~\eqref{eq:error-2}, we need first bound $\| \gram(s)^{-1} - (\gram^\infty)^{-1} \|_2 \|\target \|_2$. The following lemma bounds this norm by expanding the inverse with an infinite series.

\begin{lem}\label{lem:norm-inverse}
Under the same setting as Theorem~\ref{thm:d-convergence}, with probability at least $1 - \delta$ over the random initialization, we have
\begin{equation}
    \| \gram(k)^{-1} - \gram(0)^{-1} \|_2 \approx \| \gram(k)^{-1} - (\gram^\infty)^{-1} \|_2 = \bigo\left(\frac{n^{4/3}}{\vars^{2/3}\lambda_0^{7/3} m^{1/3} \delta} \right)
\end{equation}
\end{lem}
It is easy to show that both $\|\jacobian \|_2$ and $\| \target\|_2$ are $\bigo \left(\sqrt{n}\right)$. Plugging them back into eqn.~\eqref{eq:error-2}, we have
\begin{equation}
 \|\error_2\|_2 = \bigo \left(\frac{n^{7/3}}{\vars^{2/3}\lambda_0^{7/3} m^{1/3} \delta} \right)
\end{equation}
Similarly, we can bound $\error_3$ as follows,
\begin{equation}\label{eq:error-3}
\begin{aligned}
    \|\error_3 \|_2 \leq \sum_{s=0}^{k-1} \lr (1-\lr)^s\|\jacobian(s) - \jacobian(0)\|_2 \| (\gram^\infty)^{-1} \|_2 \|\target \|_2  = \bigo\left(\frac{n^{7/6}}{\vars^{1/3}\lambda_0^{7/6}m^{1/6}\delta^{1/2}} \right)
\end{aligned}
\end{equation}

We also bound the first term in eqn.~\eqref{eq:error-decomp}:
\begin{equation}\label{eq:core-term}
\begin{aligned}
&\quad \left\| \sum_{s=0}^{k-1} \lr \jacobian(0)^\top  (\gram^\infty)^{-1}(1-\lr)^s\target \right\|_2^2 \\
&\leq \left\|\jacobian(0)^\top  (\gram^\infty)^{-1}\target \right\|_2^2 \\
&= \target^\top (\gram^\infty)^{-1} \jacobian(0) \jacobian(0)^\top (\gram^\infty)^{-1} \target \\
&\leq \target^\top (\gram^\infty)^{-1} \target + \|\gram(0) - \gram^\infty \|_2 \|(\gram^\infty)^{-1} \|_2^2 \|\target \|_2^2 \\
&= \target^\top (\gram^\infty)^{-1} \target + \bigo \left( \frac{n^2 \sqrt{\log\frac{n}{\delta}}}{\lambda_0^{2} m^{1/2} } \right)
\end{aligned}
\end{equation}

Combining bounds~\eqref{eq:error-1}, \eqref{eq:error-2}, \eqref{eq:error-3} and \eqref{eq:core-term}, we have
\begin{equation}
\begin{aligned}
\|\weight - \weight(0)\|_2 &\leq \sqrt{y^\top (\gram^\infty)^{-1} \target} + \bigo\left(\sqrt{\frac{n^2 \sqrt{\log\frac{n}{\delta}}}{\lambda_0^{2} m^{1/2} }} \right) \\
& \quad + \bigo\left( \sqrt{\frac{n}{\lambda_0 \delta}}\vars + \frac{n^{7/6}}{\vars^{1/3}\lambda_0^{7/6}m^{1/6}\delta}\right) + \bigo\left(\frac{n^{7/3}}{\vars^{2/3}\lambda_0^{7/3} m^{1/3} \delta} \right) \\
& =  \sqrt{\target^\top (\gram^\infty)^{-1} \target} + \bigo \left( \sqrt{\frac{n}{\lambda_0 \delta}}\vars \right) + \frac{\mathrm{poly}(n, \frac{1}{\lambda_0}, \frac{1}{\delta})}{\vars^{1/3}m^{1/6}}
\end{aligned}
\end{equation}
This finishs the proof.
\end{proof}

\begin{proof}[Proof of Lemma~\ref{lem:norm-pred-error}]
Recall that in eqn.~\eqref{eq:update}, we have
\begin{equation}\label{eq:update-2}
\begin{aligned}
    \pred(k+1) - \pred(k) 
    &\leq \lr \left(\target - \pred(k)\right) + \lr \left[\jacobian(k+1) - \jacobian(k)\right] \jacobian(k)^\top\gram(k)^{-1} (\target - \pred(k)) \\
    &= \lr \left(\target - \pred(k)\right) + \mathbf{\xi}(k) 
\end{aligned}
\end{equation}
where $\left\|\mathbf{\xi}(k)\right\|_2 = \bigo \left(\lr \left(1-\frac{\lr}{2}\right)^{k} \frac{n^{7/6}}{\vars^{1/3}\lambda_0^{2/3}m^{1/6}\delta}  \right)$. Applying eqn.~\eqref{eq:update-2} recursively, we get
\begin{equation}
\begin{aligned}
    \pred(k) - \target &= (1-\lr)^k (\pred(0) - \target) + \sum_{s=0}^{k-1}(1-\lr)^s \mathbf{\xi}(k-1-s) \\
    & = -(1-\lr)^k \target + (1-\lr)^k \pred(0) + \sum_{s=0}^{k-1}(1-\lr)^s \mathbf{\xi}(k-1-s)
\end{aligned}
\end{equation}
For the second term, we have $\|(1 - \lr)^k \pred(0) \|_2 = \bigo \left((1-\lr)^k \sqrt{\frac{n}{\delta}} \vars\right)$.
For the last term, we have
\begin{equation}
\begin{aligned}
    \left\| \sum_{s=0}^{k-1}(1-\lr)^s \mathbf{\xi}(k-1-s) \right\|_2 &\leq \sum_{s=0}^{k-1} \lr \left(1 - \frac{\lr}{2}\right)^{k-1} \bigo \left(\frac{n^{7/6}}{\vars^{1/3}\lambda_0^{2/3}m^{1/6}\delta} \right)\\
    & =  \bigo \left(k\lr \left(1-\frac{\lr}{2}\right)^{k-1} \frac{n^{7/6}}{\vars^{1/3}\lambda_0^{2/3}m^{1/6}\delta} \right)
\end{aligned}
\end{equation}
\end{proof}

\begin{proof}[Proof of Lemma~\ref{lem:norm-inverse}]
Notice that $\gram^{-1} = \alpha \sum_{s=0}^{\infty} \left(\iden - \alpha \gram \right)^s$, as long as $\alpha$ is small enough so that $\iden - \alpha \gram$ is positive definite. Therefore, instead of bounding $\|\gram(s)^{-1} - (\gram^\infty)^{-1} \|_2$ directly, we can upper bound the following quantity:
\begin{equation}\label{eq:norm-series}
    \left\|\sum_{s=0}^\infty \left( \iden - \alpha\gram(k)\right)^s - \left( \iden - \alpha\gram^\infty\right)^s \right\|_2 \leq \sum_{s=0}^\infty \left\| \left( \iden - \alpha\gram(k)\right)^s - \left( \iden - \alpha\gram^\infty\right)^s \right\|_2
\end{equation}
Let $e(s)$ denote $\left\| \left( \iden - \alpha\gram(k)\right)^s - \left( \iden - \alpha\gram^\infty\right)^s \right\|_2$, we then have the following recursion:
\begin{equation}
    e(s+1) \leq \left\|\iden - \alpha\gram^\infty \right\|_2 e(s) + \left\|\alpha\gram(k) - \alpha\gram^\infty \right\|_2 \left\|\left(\iden - \alpha\gram(k)\right)^s \right\|_2
\end{equation}
Recall that $\lambda_\mathrm{min}\left( \gram^\infty \right) = \lambda_0 > 0$ and $\lambda_\mathrm{min}\left( \gram(k) \right) \geq \frac{1}{2}\lambda_0$, we have
\begin{equation}\label{eq:recursion}
    e(s+1) \leq (1-\alpha \lambda_0) e(s) + \left\|\alpha\gram(k) - \alpha\gram^\infty \right\|_2 \left(1 - \frac{1}{2}\alpha \lambda_0\right)^s
\end{equation}
Also we can easily bound the deviation of $\gram(k)$ from $\gram^\infty$ as follows,
\begin{equation}\label{eq:dist-inf}
\begin{aligned}
    \left\| \gram(k) - \gram^\infty \right\|_2 &\leq \left\| \gram(k) - \gram(0) \right\|_2 + \left\| \gram(0) - \gram^\infty \right\|_2 \\
    &= \bigo\left(\frac{n^{4/3}}{\vars^{2/3}\lambda_0^{1/3} m^{1/3} \delta } \right) + \bigo \left( \frac{n \sqrt{\log\frac{n}{\delta}}}{ m^{1/2}} \right) \\
    & = \bigo\left(\frac{n^{4/3}}{\vars^{2/3}\lambda_0^{1/3} m^{1/3} \delta } \right) \triangleq E
\end{aligned}
\end{equation}
Plugging eqn.~\eqref{eq:dist-inf} into eqn.~\eqref{eq:recursion}, we have 
\begin{equation}
    e(s+1) \leq (1-\alpha \lambda_0) e(s) + \alpha E \left(1 - \frac{1}{2}\alpha \lambda_0\right)^s
\end{equation}
With basic techniques of series theory and the fact $e(1) = \alpha E$, we have the following result:
\begin{equation}\label{eq:series}
    e(s) \leq \alpha E \left[\frac{\left(1-\frac{1}{2}\alpha \lambda_0 \right)^s}{\frac{1}{2}\alpha\lambda_0} + \left( 2-\frac{1}{\frac{1}{2}\alpha\lambda_0}\right) \left( 1 - \alpha\lambda_0 \right)^{s-1}  \right], \; \forall s \geq 1
\end{equation}
Plugging eqn.~\eqref{eq:series} back into eqn.~\eqref{eq:norm-series}, we get
\begin{equation}
\begin{aligned}
    & \quad \left\|\sum_{s=0}^\infty \left( \iden - \alpha\gram(k)\right)^s - \left( \iden - \alpha\gram^\infty\right)^s \right\|_2 \\
    &\leq \sum_{s=1}^\infty \alpha E \left[\frac{\left(1-\frac{1}{2}\alpha \lambda_0 \right)^s}{\frac{1}{2}\alpha\lambda_0} + \left( 2-\frac{1}{\frac{1}{2}\alpha\lambda_0}\right) \left( 1 - \alpha\lambda_0 \right)^{s-1}  \right] \\
    &= \frac{2E}{\alpha \lambda_0^2} = \bigo\left( \frac{n^{4/3}}{\vars^{2/3}\lambda_0^{7/3} m^{1/3} \delta  \alpha} \right)
\end{aligned}
\end{equation}
and 
\begin{equation}
    \left\|\gram(k)^{-1} - (\gram^\infty)^{-1} \right\|_2 = \bigo \left( \frac{n^{4/3}}{\vars^{2/3}\lambda_0^{7/3} m^{1/3} \delta} \right)
\end{equation}
This completes the proof.
\end{proof}

\section{Asymptotic analysis}
As shown by~\citet{lee2019wide}, infinitely wide neural networks are linearized networks in the sense that the first-order Taylor expansion is accurate and the training dynamics of wide neural networks are well captured by linearized models in practice. Assume linearity (i.e., the Jacobian matrix $\jacobian$ is constant over $\weight$), we have the following result:
\begin{equation}
\begin{aligned}
    \pred(k+1) - \pred(k) &= \pred\left(\weight(k) - \lr \jacobian(k)^\top \gram(k)^{-1} (\pred(k) - \target)\right) - \pred(\weight(k)) \\
    &= -\int_{s=0}^\lr \left\langle \frac{\partial \pred\left(\weight(k)\right)}{\partial \weight^\top}, \jacobian(k)^\top \gram(k)^{-1}(\pred(k) - \target) \right\rangle ds \\
    & \quad + \int_{s=0}^\lr \left\langle \underbrace{\frac{\partial \pred\left(\weight(k)\right)}{\partial \weight^\top} - \frac{\partial \pred\left(\weight(s)\right)}{\partial \weight^\top}}_{= \mathbf{0}}, \jacobian(k)^\top \gram(k)^{-1}(\pred(k) - \target) \right\rangle ds \\
    & = -\int_{s=0}^\lr \left\langle \jacobian(k), \jacobian(k)^\top \gram(k)^{-1}(\pred(k) - \target) \right\rangle ds 
     = \lr \left(\target - \pred(k)\right)
\end{aligned}
\end{equation}
So it is easy to show that
\begin{equation}
    \target - \pred(k+1) = (1 - \lr) \left(\target - \pred(k) \right)
\end{equation}
which means exact natural gradient descent can converge with one iteration if we take $\lr = 1$, demonstrating the effectiveness of natural gradient descent.

Moreover, under the linearized network, we can conveniently analyze the trajectories of GD and NGD. Notably, we analyze the paths taken by GD and NGD in both output space and weight space. The dynamics with infinitesimal step size are summarized as follow.
\begin{equation}
\begin{aligned}[c]
    \frac{d}{dt}\left(\target - \pred(t)\right) &=  - \gram \left(\target - \pred(t) \right) \\
    \frac{d}{dt}\weight_\mathrm{GD}(t) & = \jacobian^\top \left(\target - \pred(t) \right) \\
\end{aligned}
\qquad
\begin{aligned}[c]
    \frac{d}{dt}\left(\target - \pred(t)\right) &= - \left(\target - \pred(t) \right) \\
    \frac{d}{dt}\weight_\mathrm{NGD}(t) & = \jacobian^\top \gram^{-1} \left(\target - \pred(t) \right) \\
\end{aligned}
\end{equation}
By standard matrix differential equation theory, we have
\begin{equation}
\begin{aligned}[c]
    &\weight_\mathrm{GD}(t) = \jacobian^\top \gram^{-1}\left(\iden - \exp(-\gram t)  \right)\left(\target - \pred(0) \right) + \weight(0) \\
    &\weight_\mathrm{NGD}(t) = \left(1 - \exp(- t) \right)\jacobian^\top \gram^{-1}\left(\target - \pred(0) \right) + \weight(0)
\end{aligned}
\end{equation}
By some manipulations, we can show $\weight_\mathrm{GD}(\infty) = \weight_\mathrm{NGD}(\infty)$,
which means that gradient descent and natural gradient descent converge to the same point, though these two paths are typically different. Notably, the limiting distance $\weight(\infty) - \weight(0)$ (for both GD and exact NGD) converges to min-norm least squares solution. 